\documentclass[a4paper]{article}
\usepackage{iclr2017_conference,times}
\usepackage{amsmath}
\usepackage{amsthm}
\usepackage{amsfonts}
\usepackage{amssymb}
\usepackage{graphicx}
\usepackage{subcaption}
\usepackage{booktabs}
\usepackage{enumitem}
\usepackage{comment}
\usepackage{floatrow}
\usepackage[bottom]{footmisc}
\DeclareCaptionLabelFormat{andtable}{#1~#2  \& \tablename~\thetable}
\newfloatcommand{capbtabbox}{table}[][\FBwidth]
\usepackage{booktabs}
\iclrfinalcopy

\usepackage[colorlinks=true,allcolors=blue]{hyperref}
\usepackage[colorinlistoftodos]{todonotes}

\usepackage{xspace}
\xspaceaddexceptions{]\}}

\newtheorem{proposition}{Proposition}[section]

\newlist{inlinelist}{enumerate*}{1}
\setlist*[inlinelist,1]{%
	label=(\roman*),
}

\newcommand{\article}{paper\xspace}

\newcommand{\pd}{p_\text{data}}
\newcommand{\pg}{p_\text{gen}}
\newcommand{\pz}{p_{z}}
\newcommand{\qg}{q_\text{gen}}

\DeclareMathOperator*{\E}{\mathbb{E}}
\DeclareMathOperator{\R}{\mathbb{R}}
\DeclareMathOperator*{\argmax}{\arg\max}
\DeclareMathOperator*{\argmin}{\arg\min}

\title{Calibrating Energy-based Generative Adversarial Networks}
\author{
	Zihang Dai$^1$, Amjad Almahairi$^2$\thanks{Part of this work was completed while author was at Maluuba Research}, Philip Bachman$^3$, Eduard Hovy$^1$ \& Aaron Courville$^2$\\
    $^1$ Language Technologies Institute, Carnegie Mellon University.\\
    $^2$ MILA,  Universit\'e de Montr\'eal.\\
    $^3$ Maluuba Research.
}
\begin{document}
\maketitle
	
\begin{abstract}
In this paper we propose equipping Generative Adversarial Networks with the ability to produce direct energy estimates for samples.
Specifically, we develop a flexible adversarial training framework, and prove this framework not only ensures the generator converges to the true data distribution, but also enables the discriminator to retain the density information at the global optimum.
We derive the analytic form of the induced solution, and analyze its properties.
In order to make the proposed framework trainable in practice, we introduce two effective approximation techniques.
Empirically, the experiment results closely match our theoretical analysis, verifying that the discriminator is able to recover the energy of data distribution.
\end{abstract}

\section{Introduction}
\label{sec:introduction}

Generative Adversarial Networks (GANs)~\citep{goodfellow2014generative} represent an important milestone on the path towards more effective generative models.
GANs cast generative model training as a minimax game between a generative network (\emph{generator}), which maps a random vector into the data space, and a discriminative network (\emph{discriminator}), whose objective is to distinguish generated samples from real samples. Multiple researchers~\cite{radford2015unsupervised,salimans2016improved,zhao2016energy} have shown that the adversarial interaction with the discriminator can result in a generator that produces compelling samples. 
The empirical successes of the GAN framework were also supported by the theoretical analysis of \citeauthor{goodfellow2014generative}, who showed that,
under certain conditions, the distribution produced by the generator converges to the true data distribution, while
the discriminator converges to a degenerate uniform solution.


While GANs have excelled as compelling sample generators, their use as general purpose probabilistic generative models has been limited by the difficulty in using them to provide density estimates or even unnormalized energy values for sample evaluation. 

It is tempting to consider the GAN discriminator as a candidate for providing this sort of scoring function. Conceptually, it is a trainable sample evaluation mechanism that -- owing to GAN training paradigm -- could be closely calibrated to the distribution modeled by the generator. 
If the discriminator could retain fine-grained information of the relative quality of samples, measured for instance by probability density or unnormalized energy, it could be used as an evaluation metric.
Such data-driven evaluators would be highly desirable for problems where it is difficult to define evaluation criteria that correlate well with human judgment.
Indeed, the real-valued discriminator of the recently introduced energy-based GANs~\cite{zhao2016energy} might seem like an ideal candidate energy function.
Unfortunately, as we will show, the degenerate fate of the GAN discriminator at the optimum equally afflicts the energy-based GAN of \citeauthor{zhao2016energy}.


In this paper we consider the questions:
\begin{inlinelist}
	\item does there exists an adversarial framework that induces a non-degenerate discriminator, and
	\item if so, what form will the resulting discriminator take? 
\end{inlinelist}
We introduce a novel adversarial learning formulation, which leads to a non-degenerate discriminator while ensuring the generator distribution matches the data distribution at the global optimum. 
We derive a general analytic form of the optimal discriminator, and discuss its properties and their relationship to the specific form of the training objective.
We also discuss the connection between the proposed formulation and existing alternatives such as the approach of~\cite{kim2016deep}.
Finally, for a specific instantiation of the general formulation, we investigate two approximation techniques to optimize the training objective, and verify our results empirically.

\section{Related Work}
\label{sec:related_work}

Following a similar motivation, the field of Inverse Reinforcement Learning (IRL)~\citep{ng2000algorithms} has been exploring ways to recover the ``intrinsic'' reward function (analogous to the discriminator) from observed expert trajectories (real samples).
Taking this idea one step further, apprenticeship learning or imitation learning~\citep{abbeel2004apprenticeship,ziebart2008maximum} aims at learning a policy (analogous to the generator) using the reward signals recovered by IRL.
Notably, \citeauthor{ho2016generative} draw a connection between imitation learning and GAN by showing that the GAN formulation can be derived by imposing a specific regularization on the reward function.
Also, under a special case of their formulation, \citeauthor{ho2016generative} provide a duality-based interpretation of the problem, which inspires our theoretical analysis.
However, as the focus of \citep{ho2016generative} is only on the policy, the authors explicitly propose to bypass the intermediate IRL step, and thus provide no analysis of the learned reward function.

The GAN models most closely related to our proposed framework are energy-based GAN models of \citet{zhao2016energy} and \citet{kim2016deep}. In the next section, We show how one can derive both of these approaches from different assumptions regarding regularization of the generative model. 

\section{Alternative Formulation of Adversarial Training}
\label{sec:formulation}

\subsection{Background}
Before presenting the proposed formulation, we first state some basic assumptions required by the analysis, and introduce notations used throughout the paper.

Following the original work on GANs~\citep{goodfellow2014generative}, our analysis focuses on the non-parametric case, where all models are assumed to have infinite capacities.
While many of the non-parametric intuitions can directly transfer to the parametric case, we will point out cases where this transfer fails.
We assume a finite data space throughout the analysis, to avoid technical machinery out of the scope of this paper.
Our results, however, can be extended to continuous data spaces, and our experiments are indeed performed on continuous data.

Let $\mathcal{X}$ be the data space under consideration, and $\mathcal{P} = \{ p \mid p(x) \geq 0, \forall x \in \mathcal{X}, \sum_{x \in \mathcal{X}} p(x) = 1 \}$ be the set of all proper distributions defined on $\mathcal{X}$.
Then, $\pd \in \mathcal{P}: \mathcal{X} \mapsto \R$ and $\pg \in \mathcal{P}: \mathcal{X} \mapsto \R$ will denote the true data distribution and the generator distribution.
$\E_{x \sim p} f(x)$ denotes the expectation of the quantity $f(x)$ w.r.t.~$x$ drawn from $p$.
Finally, the term ``discriminator'' will refer to any structure that provides training signals to the generator based on some measure of difference between the generator distribution and the real data distribution, which which includes but is not limited to $f$-divergence.

\subsection{Proposed Formulation}
In order to understand the motivation of the proposed approach, it is helpful to analyze the optimization dynamics near convergence in GANs first.

When the generator distribution matches the data distribution, 
the training signal (gradient) w.r.t. the discriminator vanishes.
At this point, assume the discriminator still retains density information, and views some samples as more real and others as less.
This discriminator will produce a training signal (gradient) w.r.t. the generator, pushing the generator to generate samples that appear more real to the discriminator.
Critically, this training signal is the sole driver of the generator's training. 
Hence, the generator distribution will diverge from the data distribution.
In other words, as long as the discriminator retains relative density information, the generator distribution cannot stably match the data distribution.
Thus, in order to keep the generator stationary as the data distribution, the discriminator must assign flat (exactly the same) density to all samples at the optimal. 

From the analysis above, the fundamental difficulty is that the generator only receives a single training signal (gradient) from the discriminator, which it has to follow.
To keep the generator stationary, this single training signal (gradient) must vanish, which requires a degenerate discriminator. 
In this work, we propose to tackle this single training signal constraint directly.
Specifically, we introduce a novel adversarial learning formulation which incorporates an additional training signal to the generator, such that this additional signal can 
\begin{itemize}
	\item balance (cancel out) the discriminator signal at the optimum, so that the generator can stay stationary even if the discriminator assigns non-flat density to samples
	\item cooperate with the discriminator signal to make sure the generator converges to the data distribution, and the discriminator retains the \textit{correct} relative density information 
\end{itemize} 


The proposed formulation can be written as the following minimax training objective, 
\begin{equation}
\label{eq:general_objective}
\max_{c} \min_{\pg \in \mathcal{P}} \quad
	\E_{x \sim \pg} \big[ c(x) \big] - 
    \E_{x \sim \pd} \big[ c(x) \big] + K(\pg),
\end{equation}
where $c(x): \mathcal{X} \mapsto \R$ is the discriminator that assigns each data point an unbounded scalar cost, and $K(\pg): \mathcal{P} \mapsto \R$ is some (functionally) differentiable, convex function of $\pg$.
Compared to the original GAN, despite the similar minimax surface form, the proposed fomulation has two crucial distinctions. 

Firstly, while the GAN discriminator tries to distinguish ``fake'' samples from real ones using binary classification, the proposed discriminator achieves that by assigning lower cost to real samples and higher cost to ``fake'' one.
This distinction can be seen from the first two terms of Eqn. \eqref{eq:general_objective}, where the discriminator $c(x)$ is trained to widen the expected cost gap between ``fake'' and real samples, while
the generator is adversarially trained to minimize it. 
In addition to the different adversarial mechanism, a calibrating term $K(\pg)$ is introduced to provide a countervailing source of training signal for $\pg$ as we motivated above.
For now, the form of $K(\pg)$ has not been specified. But as we will see later, its choice will directly decide the form of the optimal discriminator $c^*(x)$.


With the specific optimization objective, we next provide theoretical characterization of both the generator and the discriminator at the global optimum. 

Define $L(\pg, c) = \E_{x \sim \pg} \big[ c(x) \big] - \E_{x \sim \pd} \big[ c(x) \big] + K(\pg)$, then $L(\pg, c)$ is the Lagrange dual function of the following optimization problem
\begin{equation}
\label{eq:constrained_primal_problem}
\begin{aligned}
	\min_{\pg \in \mathcal{P}}&\quad  K(\pg) \\
	\text{s.t.}&\quad \pg(x) - \pd(x) = 0, \forall x \in \mathcal{X}
\end{aligned}
\end{equation}
where $c(x), \forall x$ appears in $L(\pg, c)$ as the dual variables introduced for the equality constraints.
This duality relationship has been observed previously in \citep[equation (7)]{ho2016generative} under the adversarial imitation learning setting.
However, in their case, the focus was fully on the generator side (induced policy), and no analysis was provided for the discriminator (reward function).

In order to characterize $c^*$, we first expand the set constraint on $\pg$ into explicit equality and inequality constraints:
\begin{equation}
\label{eq:primal_problem}
\begin{aligned}
	\min_{\pg}
    	&\quad K(\pg) \\
	\text{s.t.}
   		&\quad \pg(x) - \pd(x) = 0, \forall x \\
		&\quad -\pg(x) \leq 0, \forall x \\
		&\quad \sum_{x \in \mathcal{X}} \pg(x) - 1 = 0.
\end{aligned}
\end{equation}

Notice that $K(\pg)$ is a convex function of $\pg(x)$ by definition, and both the equality and inequality constraints are affine functions of $\pg(x)$.
Thus, problem \eqref{eq:constrained_primal_problem} is a convex optimization problem.
What's more, since 
\begin{inlinelist}
\item $\text{dom}_K$ is open, and
\item there exists a feasible solution $\pg = \pd$ to \eqref{eq:primal_problem},
\end{inlinelist}
by the refined Slater's condition~\citep[page~226]{boyd2004convex}, 
we can further verify that strong duality holds for \eqref{eq:primal_problem}.
With strong duality, a typical approach to characterizing the optimal solution is to apply the Karush-Kuhn-Tucker (KKT) conditions, which gives rise to this theorem:

\begin{proposition}
\label{thm:general_optimal_disc}

By the KKT conditions of the convex problem \eqref{eq:primal_problem}, at the global optimum, the optimal generator distribution $\pg^*$ matches the true data distribution $\pd$, and the optimal discriminator $c^*(x)$ has the following form:
\begin{equation}
\label{eq:optimal_disc}
\begin{aligned}
c^*(x) &= -\frac{\partial K(\pg)}{\partial \pg(x)}\bigg|_{\pg=\pd} - \lambda^* + \mu^*(x), \forall x \in \mathcal{X}, \\
\text{where}
	&\quad \mu^*(x) =  
		\begin{cases}
		0, & \pd(x) > 0 \\
		u_x, & \pd(x) = 0
		\end{cases}, \\
	&\quad \lambda^* \in \R,\text{ is an under-determined real number independent of } x, \\
	&\quad u_x \in \R_+,\text{ is an under-determined non-negative real number.}
\end{aligned}
\end{equation}
	
\end{proposition}

The detailed proof of proposition \ref{thm:general_optimal_disc} is provided in appendix \ref{sec:proof_for_formulation}. 
From \eqref{eq:optimal_disc}, we can see the exact form of the optimal discriminator depends on the term $K(\pg)$, or more specifically its gradient.
But, before we instantiate $K(\pg)$ with specific choices and show the corresponding forms of $c^*(x)$, we first discuss some general properties of $c^*(x)$ that do not depend on the choice of $K$.

\textbf{Weak Support Discriminator.} As part of the optimal discriminator function, the term $\mu^*(x)$ plays the role of support discriminator. 
That is, it tries to distinguish the support of the data distribution, i.e.~$\textsc{supp}(\pd) = \{x \in \mathcal{X} \mid \pd(x) > 0\}$, from its complement set with zero-probability, i.e.~$\textsc{supp}(\pd)^\complement = \{x \in \mathcal{X} \mid \pd(x) = 0\}$.
Specifically, for any $x \in \textsc{supp}(\pd)$ and $x^\prime \in \textsc{supp}(\pd)^\complement$, it is guaranteed that $\mu^*(x) \leq \mu^*(x^\prime)$.
However, because $\mu^*(\cdot)$ is under-determined, there is nothing preventing the inequality from degenerating into an equality.
Therefore, we name it the \textit{weak} support discriminator.
But, in all cases, $\mu^*(\cdot)$ assigns zero cost to all data points within the support.
As a result, it does not possess any fine-grained density information inside of the data support.
It is worth pointing out that, in the parametric case, because of the smoothness and the generalization properties of the parametric model, the learned discriminator may generalize beyond the data support.

\textbf{Global Bias.} In \eqref{eq:optimal_disc}, the term $\lambda^*$ is a scalar value shared for all $x$.
As a result, it does not affect the relative cost among data points, and only serves as a global bias for the discriminator function.

Having discussed general properties, we now consider some specific cases of the convex function $K$, and analyze the resulting optimal discriminator $c^*(x)$ in detail.
\begin{enumerate}[leftmargin=16pt,labelindent=16pt]
\item First, let us consider the case where $K$ is the negative entropy of the generator distribution, i.e. $K(\pg) = -H(\pg)$. 
Taking the derivative of the negative entropy w.r.t. $\pg(x)$, we have 
\begin{equation}
\label{eq:optimal_disc_entropy}
\begin{aligned}
	c_\text{ent}^*(x) &= -\log \pd(x) - 1 - \lambda^* + \mu^*(x), \forall x \in \mathcal{X}, \\
\end{aligned}
\end{equation}
where $\mu^*(x)$ and $\lambda^*$ have the same definitions as in \eqref{eq:optimal_disc}.

Up to a constant, this form of $c_\text{ent}^*(x)$ is exactly the energy function of the data distribution $\pd(x)$.
This elegant result has deep connections to several existing formulations, which include max-entropy imitation learning~\citep{ziebart2008maximum} and the directed-generator-trained energy-based model~\citep{kim2016deep}.
The core difference is that these previous formulations are originally derived from maximum-likelihood estimation, and thus the minimax optimization is only implicit.
In contrast, with an explicit minimax formulation we can develop a better understanding of the induced solution.
For example, the global bias $\lambda^*$ suggests that there exists more than one stable equilibrium the optimal discriminator can actually reach.
Further, $\mu^*(x)$ can be understood as a support discriminator that poses extra cost on generator samples which fall in zero-probability regions of data space.

\item When $K(\pg) = \frac{1}{2}\sum_{x \in \mathcal{X}} \pg(x)^2 = \frac{1}{2}\|\pg\|_2^2$, which can be understood as posing $\ell_2$ regularization on $\pg$, we have $\frac{\partial K(\pg)}{\partial \pg(x)}\big|_{\pg=\pd} = \pd(x)$, and it follows
\begin{equation}
\label{eq:optimal_disc_l2}
	c_{\ell_2}^*(x) = - \pd(x) - \lambda^* + \mu^*(x), \forall x \in \mathcal{X},
\end{equation}
with $\mu^*(x), \lambda^*$ similarly defined as in \eqref{eq:optimal_disc}.

Surprisingly, the result suggests that the optimal discriminator $c_{\ell_2}^*(x)$ directly recovers the negative probability $-\pd(x)$, shifted by a constant.
Thus, similar to the entropy solution \eqref{eq:optimal_disc_entropy}, it fully retains the relative density information of data points within the support.

However, because of the under-determined term $\mu^*(x)$, we cannot recover the distribution density $\pd$ exactly from either $c_{\ell_2}^*$ or $c_\text{ent}^*$ if the data support is finite.
Whether this ambiguity can be resolved is beyond the scope of this \article, but poses an interesting research problem. 


\item Finally, let's consider consider a degenerate case, where $K(\pg)$ is a constant. 
That is, we don’t provide
any additional training signal for pgen at all. With $K(\pg) = \text{const}$, we simply have
\begin{equation}
c_\text{cst}^*(x) = −\lambda^* + \mu^*(x), \forall x \in \mathcal{X}, 
\end{equation}
whose discriminative power is fully controlled by the weak support discriminator $\mu^*(x)$. 
Thus, it follows that $c_\text{cst}^*(x)$ won't be able to discriminate data points within the support of $\pd$, and its power to distinguish data from $\textsc{supp}(\pd)$ and $\textsc{supp}(\pd)^\complement$ is weak.
This closely matches the intuitive argument in the beginning of this section.



Note that when $K(\pg)$ is a constant, the objective function \eqref{eq:general_objective} simplifies to: 
\begin{equation}
\label{eq:constant_objective}
\max_{c} \min_{\pg \in \mathcal{P}} \quad
	\E_{x \sim \pg} \big[ c(x) \big] - 
    \E_{x \sim \pd} \big[ c(x) \big],
\end{equation}
which is very similar to the EBGAN objective~\citep[equation (2) and (4)]{zhao2016energy}.
As we show in appendix \ref{sec:ebgan_proof}, compared to the objective in \eqref{eq:constant_objective}, the EBGAN objective puts extra constraints on the allowed discriminator function.
In spite of that, the EBGAN objective suffers from the single-training-signal problem and does not guarantee that the discriminator will recover the real energy function (see appendix \ref{sec:ebgan_proof} for detailed analysis).
\end{enumerate}

As we finish the theoretical analysis of the proposed formulation, we want to point out that simply adding the same term $K(\pg)$ to the original GAN formulation will not lead to both a generator that matches the data distribution, and a discriminator that retains the density information (see appendix \ref{sec:fgan_reg} for detailed analysis).

\section{Parametric Instantiation with Entropy Approximation}
\label{sec:instantiation}
While the discussion in previous sections focused on the non-parametric case, in practice we are limited to a finite amount of data, and the actual problem involves high dimensional continuous spaces.
Thus, we resort to parametric representations for both the generator and the discriminator.
In order to train the generator using standard back-propagation, we do not parametrize the generator distribution directly. Instead, we parametrize a directed generator network that transforms random noise $z\sim \pz(z)$ to samples from a continuous data space $\R^{n}$.
Consequently, we don't have analytical access to the generator distribution, which is defined implicitly by the generator network's noise$\rightarrow$data mapping.
However, the regularization term $K(\pg)$ in the training objective~\eqref{eq:general_objective} requires the generator distribution.
Faced with this problem, we focus on the max-entropy formulation, and exploit two different approximations of the regularization term $K(\pg)=-H(\pg)$.

\subsection{Nearest-Neighbor Entropy Gradient Approximation}
The first proposed solution is built upon an intuitive interpretation of the entropy gradient.
Firstly, since we construct $\pg$ by applying a deterministic, differentiable transform $g_{\theta}$ to samples $z$ from a fixed distribution $\pz$, we can write the gradient of $H(\pg)$ with respect to the generator parameters $\theta$ as follows:
\begin{equation}
-\nabla_{\theta} H(\pg) = \mathbb{E}_{z \sim p_z} \left[ \nabla_{\theta} \log \pg(g_{\theta}(z)) \right] = \mathbb{E}_{z \sim p_z} \left[ \frac{\partial g_{\theta}(z)}{\partial \theta} \frac{\partial \log \pg(g_{\theta}(z))}{\partial g_{\theta}(z)} \right],
\label{eq:ent_grad_wrt_gen_params}
\end{equation}
where the first equality relies on the ``reparametrization trick''. 
Equation ~\ref{eq:ent_grad_wrt_gen_params} implies that, if we can compute the gradient of the generator log-density $\log \pg(x)$ w.r.t. any $x = g_{\theta}(z)$, then we can directly construct the Monte-Carlo estimation of the entropy gradient $\nabla_{\theta} H(\pg)$ using samples from the generator.

Intuitively, for any generated data $x = g_{\theta}(z)$, the term $\frac{\partial \log \pg(x)}{\partial x}$ essentially describes the direction of \textit{local change}
in the sample space that will increase the log-density.
Motivated by this intuition, we propose to form a local Gaussian approximation $\pg^i$ of $\pg$ around each point $x_i$ in a batch of samples $\{x_1, ..., x_n\}$ from the generator, and then compute the gradient $\frac{\partial \log \pg(x_i)}{\partial x_i}$ based on the Gaussian approximation. 
Specifically, each local Gaussian approximation $\pg^i$ is formed by finding the $k$ nearest neighbors of $x_i$ in the batch $\{x_1, ..., x_n\}$, and then placing an isotropic Gaussian distribution at their mean (i.e.~maximimum likelihood).
Based on the isotropic Gaussian approximation, the resulting gradient has the following form
\begin{equation}
\label{eq:knn_approx}
\frac{\partial \log \pg(x_i)}{\partial x_i} \approx \mu_i - x_i, \quad \text{ where } \mu_i = \frac{1}{k} \sum_{x' \in \text{KNN}(x_i)} x' \text{ is the mean of the Gaussian}
\end{equation}
Finally, note the scale of this gradient approximation may not be reliable.
To fix this problem, we normalize the approximated gradient into unit norm, and use a single hyper-parameter to model the scale for all $x$, leading to the following entropy gradient approximation
\begin{equation}
-\nabla_{\theta} H(\pg) \approx \alpha \frac{1}{k} \sum_{x_i = g_{\theta}(z_i)} \frac{\mu_i - x_i}{\|\mu_i - x_i\|_2}
\end{equation}
where $\alpha$ is the hyper-parameter and $\mu_i$ is defined as in equation \eqref{eq:knn_approx}.

An obvious weakness of this approximation is that it relies on Euclidean distance to find the $k$ nearest neighbors.
However, Euclidean distance is usually not the proper metric to use when the effective dimension is very high.
As the problem is highly challenging, we leave it for future work.

\subsection{Variational Lower bound on the Entropy}
Another approach we consider relies on defining and maximizing a variational lower bound on the entropy $H(\pg(x))$ of the generator distribution.
We can define the joint distribution over observed data and the noise variables as $\pg(x,z)=\pg(x\mid z)\pg(z)$, where simply $\pg(z)=\pz(z)$ is a fixed prior. Using the joint, we can also define the marginal $\pg(x)$ and the posterior $\pg(z\mid x)$.
We can also write the mutual information between the observed data and noise variables as:
\begin{equation}
\begin{aligned}
I(\pg(x);\pg(z)) &= H(\pg(x))-H(\pg(x\mid z))\\
&= H(\pg(z))-H(\pg(z\mid x)),
\end{aligned}
\end{equation}
where $H(\pg(.\mid .))$ denotes the conditional entropy. By reorganizing terms in this definition, we can write the entropy $H(\pg(x))$ as:
\begin{equation}
H(\pg(x))=H(\pg(z))-H(\pg(z\mid x))+H(\pg(x\mid z))
\end{equation}
We can think of $\pg (x\mid z)$ as a peaked Gaussian with a fixed, diagonal covariance, and hence its conditional entropy is constant and can be dropped. Furthermore, $H(\pg(z))$ is also assumed to be fixed a priori. Hence, we can maximize $H(\pg(x))$ by minimizing the conditional entropy:
\begin{equation}
H(\pg(z\mid x))=\mathbb{E}_{x\sim \pg (x)}\left[\mathbb{E}_{z\sim \pg(z\mid x)}\left[-\log \pg(z\mid x)\right]\right]
\end{equation}
Optimizing this term is still problematic, because
\begin{inlinelist}
\item we do not have access to the posterior $\pg(z\mid x)$, and
\item we cannot sample from it.
\end{inlinelist}~Therefore, we instead minimize a variational upper bound defined by an approximate posterior $\qg(z\mid x)$:
\begin{equation}
\begin{aligned}
H(\pg(z\mid x)) &=  \mathbb{E}_{x\sim \pg(x)}\left[\mathbb{E}_{z\sim \pg(z\mid x)}\left[-\log \qg(z\mid x)\right] - 
\textrm{KL}(\pg(z\mid x)\| \qg(z\mid x))\right] \\
&\leq  \mathbb{E}_{x\sim \pg(x)}\left[\mathbb{E}_{z\sim \pg(z\mid x)}\left[-\log \qg(z\mid x)\right]\right] \\
&= \mathcal{U}(\qg).
\end{aligned}
\end{equation}
We can also rewrite the variational upper bound as:
\begin{equation}
\mathcal{U}(\qg) = \mathbb{E}_{x,z\sim \pg(x,z)}\left[-\log \qg(z\mid x)\right]=\mathbb{E}_{z\sim \pg(z)}\left[\mathbb{E}_{x\sim \pg(x\mid z)}\left[-\log \qg(z\mid x)\right]\right],
\end{equation}
which can be optimized efficiently with standard back-propagation and Monte Carlo integration of the relevant expectations based on independent samples drawn from the joint $\pg(x, z)$. By minimizing this upper bound on the conditional entropy
$H(\pg(z\mid x))$, we are effectively maximizing a variational lower bound on the entropy $H(\pg(x))$.
%
%
%
%
\section{Experiments}
\label{sec:experiments}
In this section, we verify our theoretical results empirically on several synthetic and real datasets.
In particular, we evaluate whether the discriminator obtained from the entropy-regularized adversarial training can capture the density information (in the form of energy), while making sure the generator distribution matches the data distribution. 
For convenience, we refer to the obtained models as EGAN-Ent.
Our experimental setting follows closely recommendations from~\cite{radford2015unsupervised}, except in Sec.~\ref{sec:synthetic_exp} where we use fully-connected models (see appendix \ref{sec:experiment_setting} for details).~\footnote{For more details, please refer to \url{https://github.com/zihangdai/cegan_iclr2017}.}



\subsection{Synthetic low-dimensional data}
\label{sec:synthetic_exp}
First, we consider three synthetic datasets in $2$-dimensional space,
which are drawn from the following distributions:
\begin{inlinelist}
\item Mixture of 4 Gaussians with equal mixture weights,
\item Mixture of 200 Gaussians arranged as two spirals (100 components each spiral), and 
\item Mixture of 2 Gaussians with highly biased mixture weights, $P(c_1) = 0.9, P(c_2) = 0.1$.
\end{inlinelist}
We visualize the ground-truth energy of these distributions along with 100K training samples in Figure \ref{fig:synthetic_distributions}.
\begin{figure}[t!]
	\centering
	\begin{subfigure}{0.325\textwidth}
		\includegraphics[width=\textwidth]{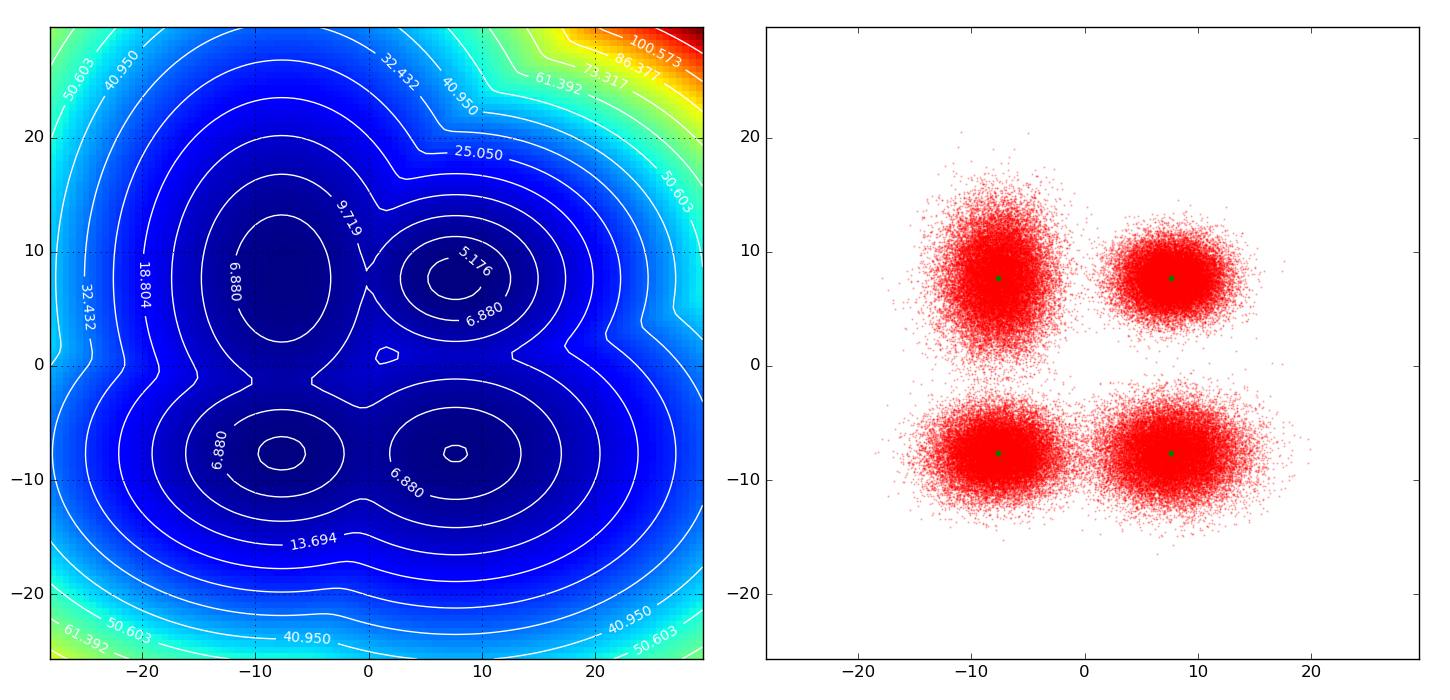}
	\end{subfigure}
    \begin{subfigure}{0.325\textwidth}
		\includegraphics[width=\textwidth]{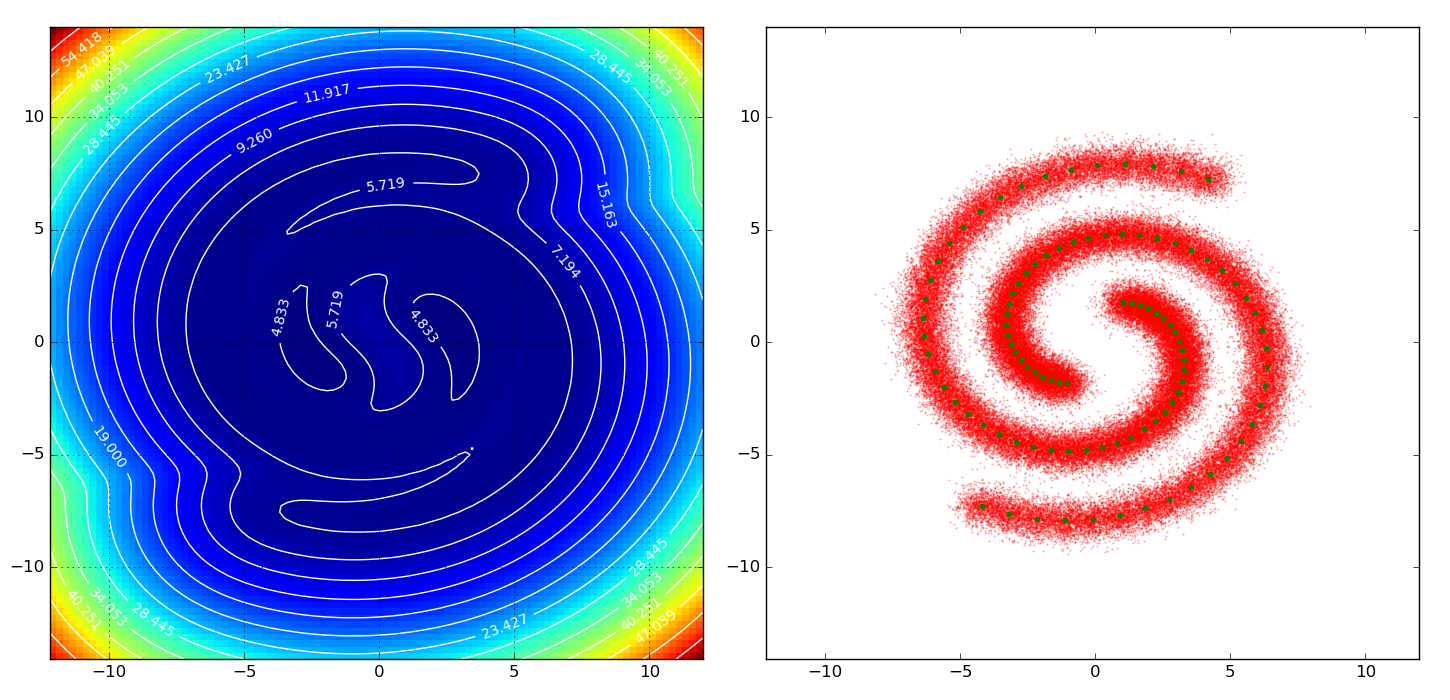}
	\end{subfigure}
    \begin{subfigure}{0.325\textwidth}
		\includegraphics[width=\textwidth]{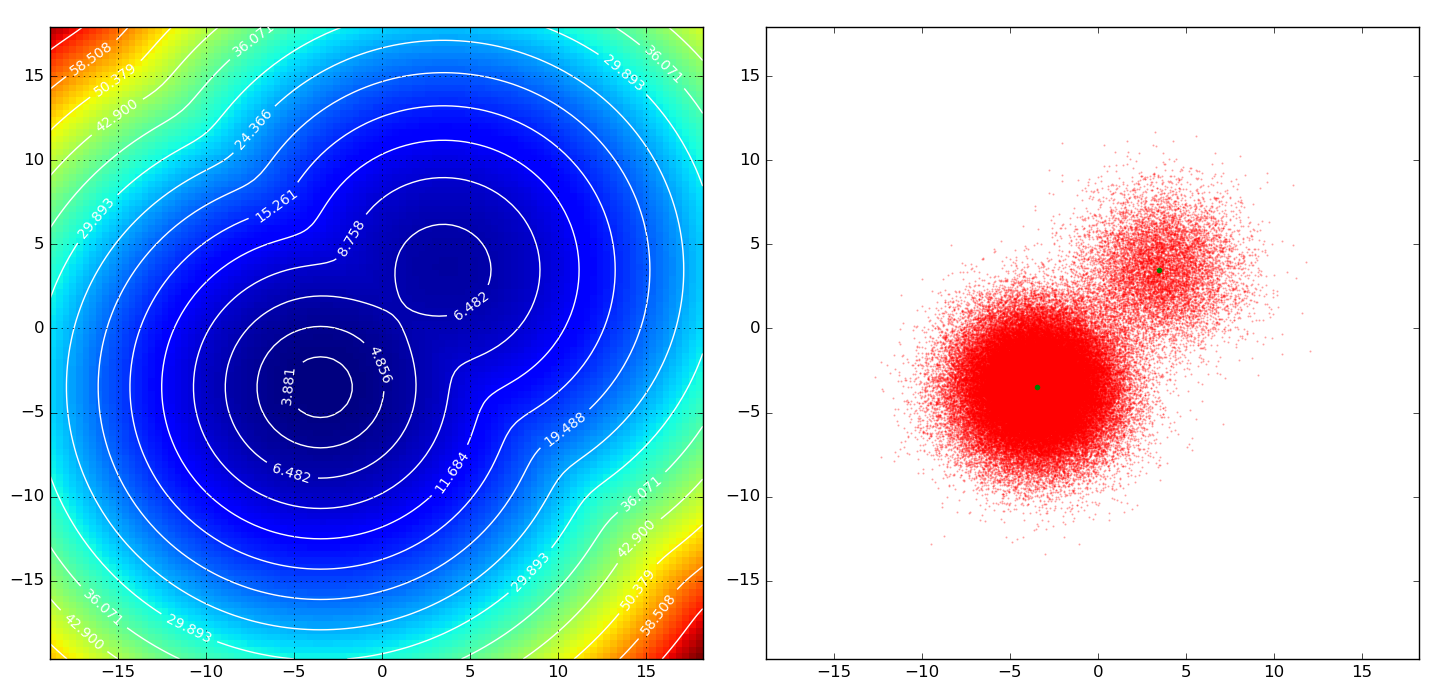}
	\end{subfigure}
	\caption{True energy functions and samples from synthetic distributions. Green dots in the sample plots indicate the mean of each Gaussian component.}
	\label{fig:synthetic_distributions}
\vspace{-0.6em}
\end{figure}
Since the data lies in $2$-dimensional space, we can easily visualize both the learned generator (by drawing samples) and the discriminator for direct comparison and evaluation. 
We evaluate here our EGAN-Ent model using both approximations: the nearest-neighbor based approximation (EGAN-Ent-NN) and the variational-inference based approximation (EGAN-Ent-VI), and compare them with two baselines: the original GAN and the energy based GAN with no regularization (EGAN-Const).

Experiment results are summarized in Figure \ref{fig:synthetic_result_1} for baseline models, and Figure \ref{fig:synthetic_result_2} for the proposed models. 
As we can see, all four models can generate perfect samples. 
However, for the discriminator, both GAN and EGAN-Const lead to degenerate solution, assigning flat energy inside the empirical data support.
In comparison, EGAN-Ent-VI and EGAN-Ent-NN clearly capture the density information, though to different degrees.
Specifically, on the equally weighted Gaussian mixture and the two-spiral mixture datasets, EGAN-Ent-NN tends to give more accurate and fine-grained solutions compared to EGAN-Ent-VI.
However, on the biased weighted Gaussian mixture dataset, EGAN-Ent-VI actually fails to captures the correct mixture weights of the two modes, incorrectly assigning lower energy to the mode with lower probability (smaller weight).
In contrast, EGAN-Ent-NN perfectly captures the bias in mixture weight, and obtains a contour very close to the ground truth.
\begin{figure}[t!]
	\centering
	\begin{subfigure}{\textwidth}
		\includegraphics[width=0.32\textwidth]{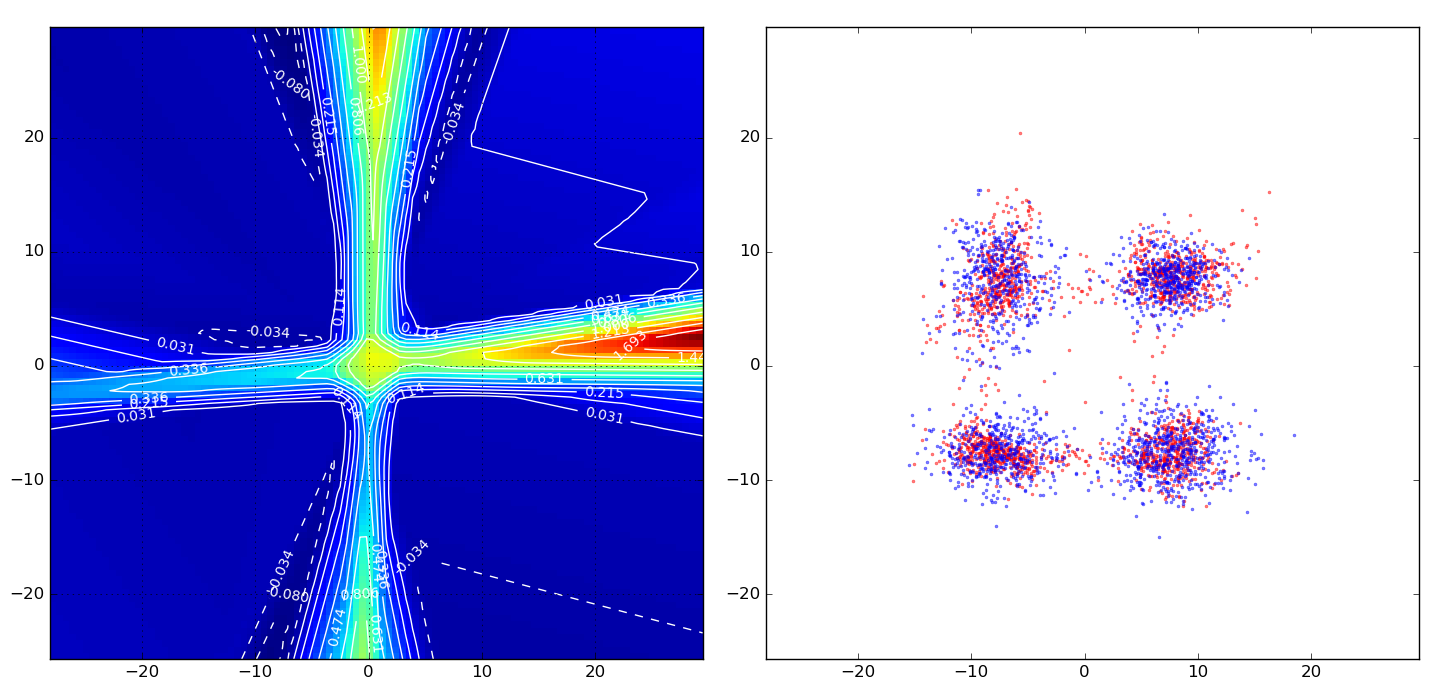}
		\includegraphics[width=0.32\textwidth]{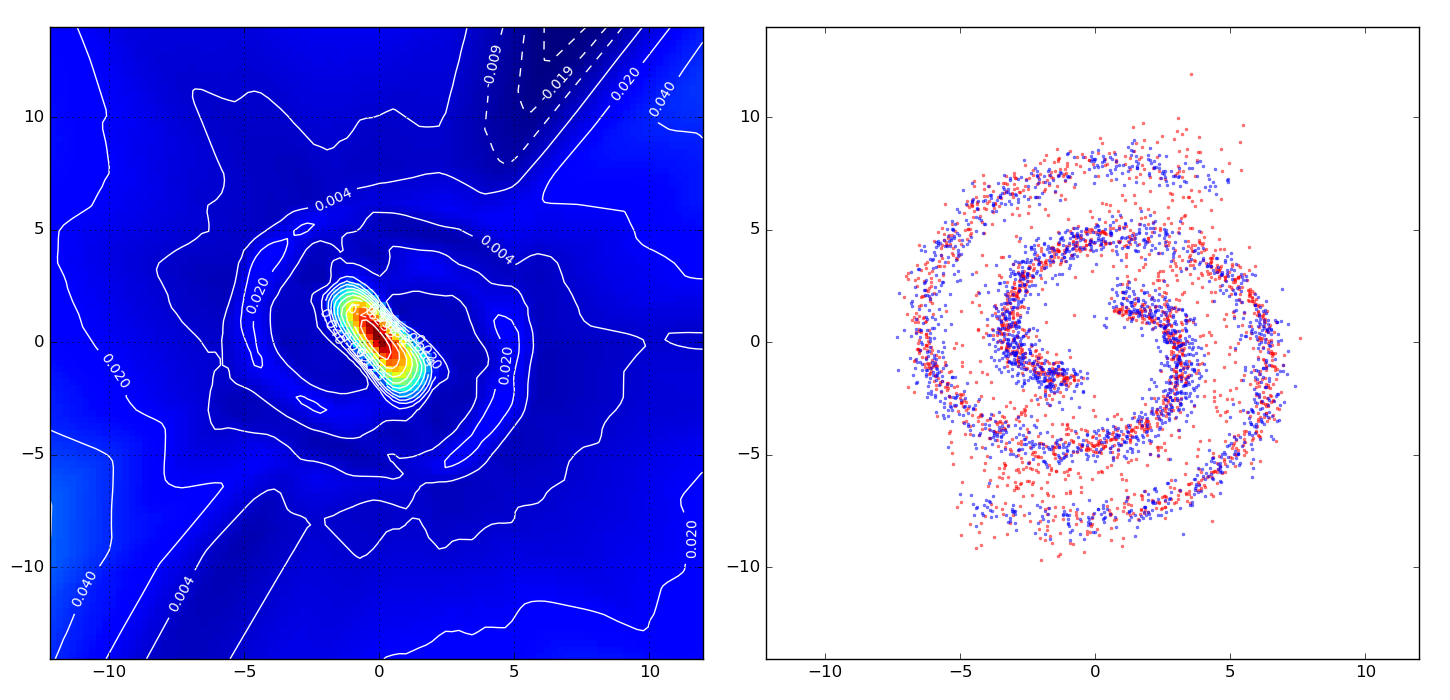}
		\includegraphics[width=0.32\textwidth]{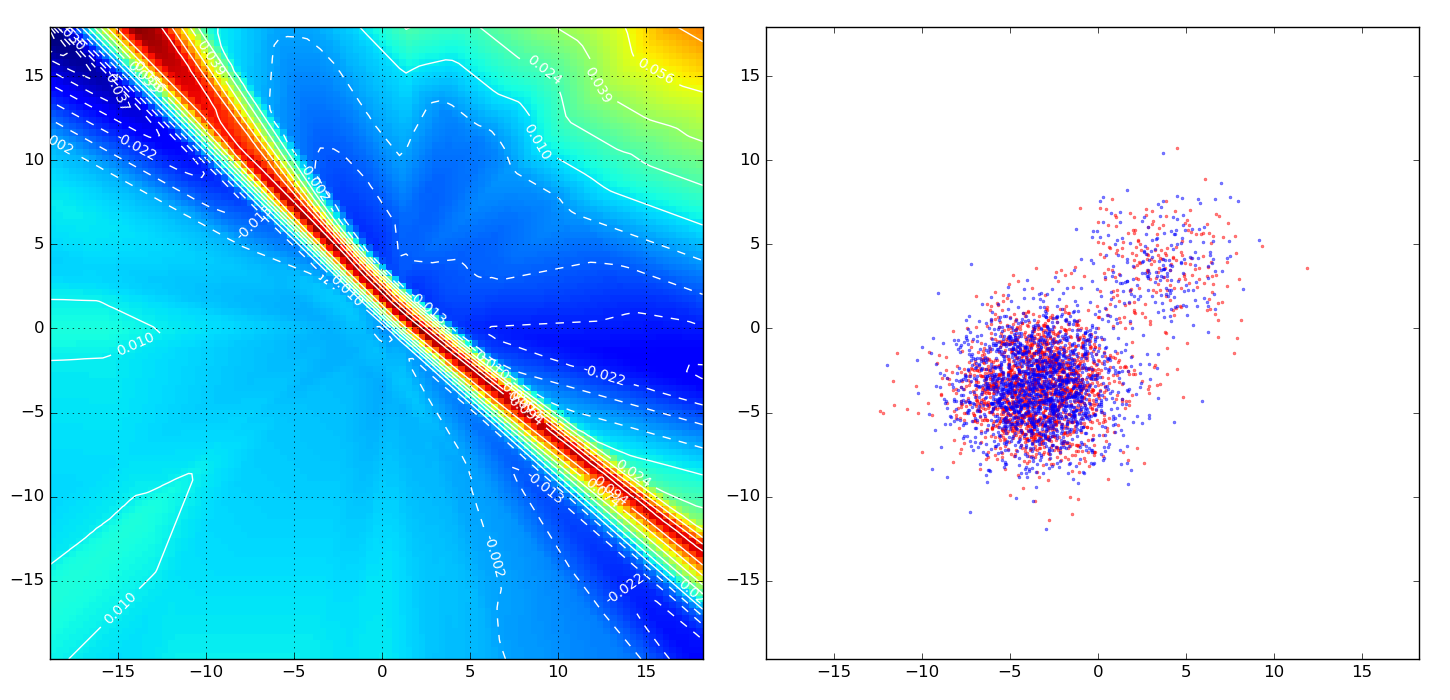}
        \caption{Standard GAN}
	\end{subfigure}
	\begin{subfigure}{\textwidth}
		\includegraphics[width=0.32\textwidth]{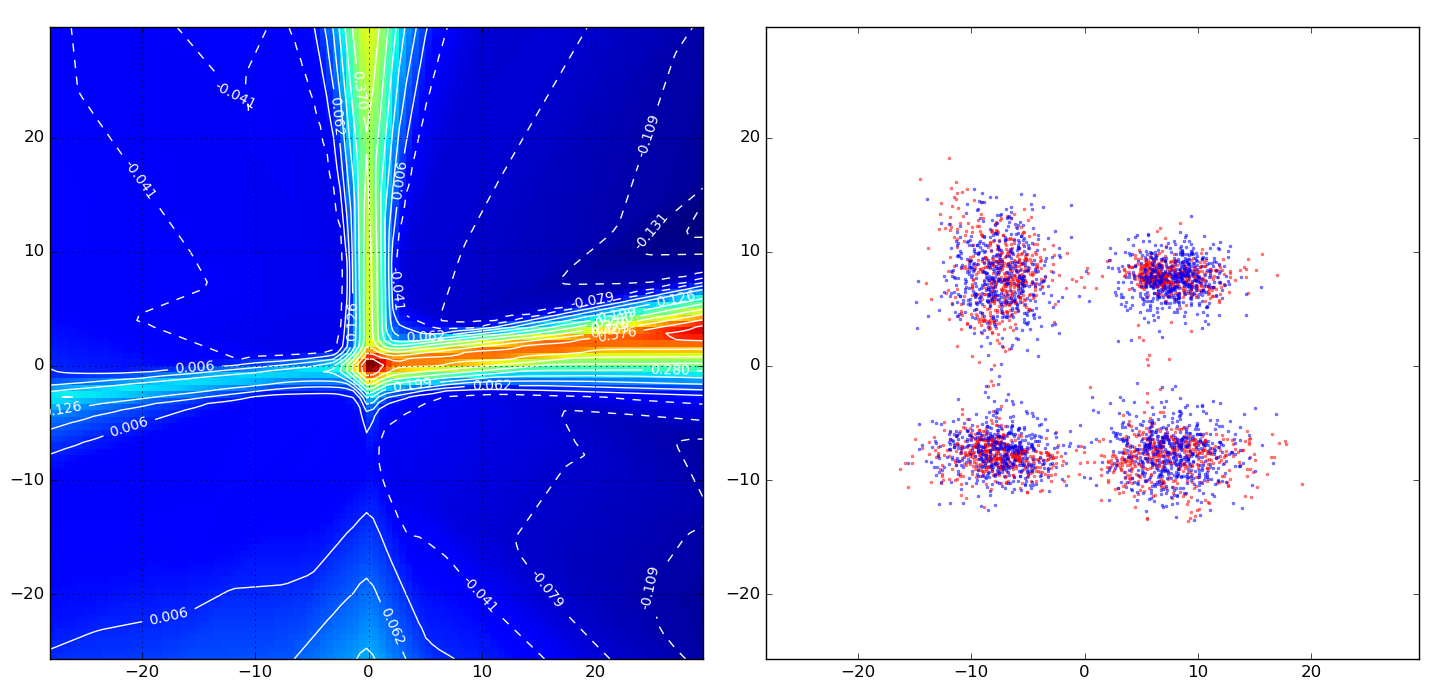}
		\includegraphics[width=0.32\textwidth]{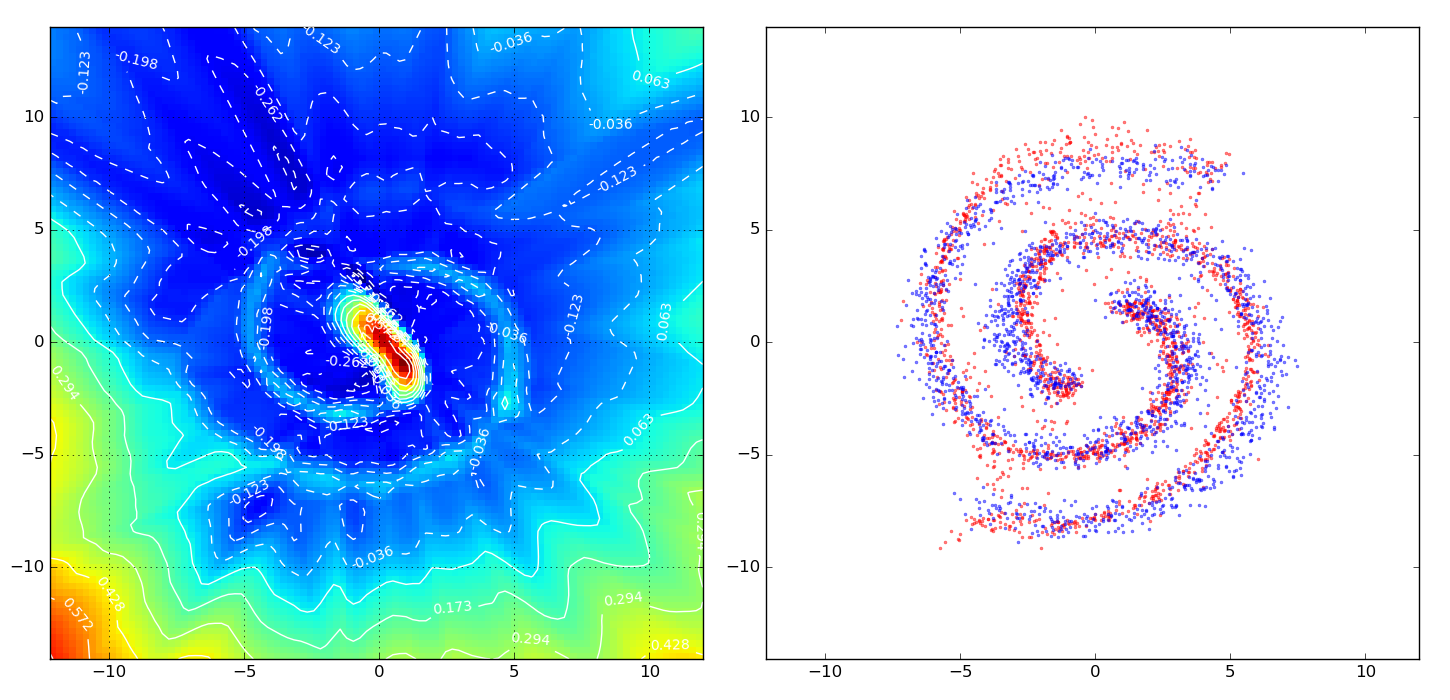}
		\includegraphics[width=0.32\textwidth]{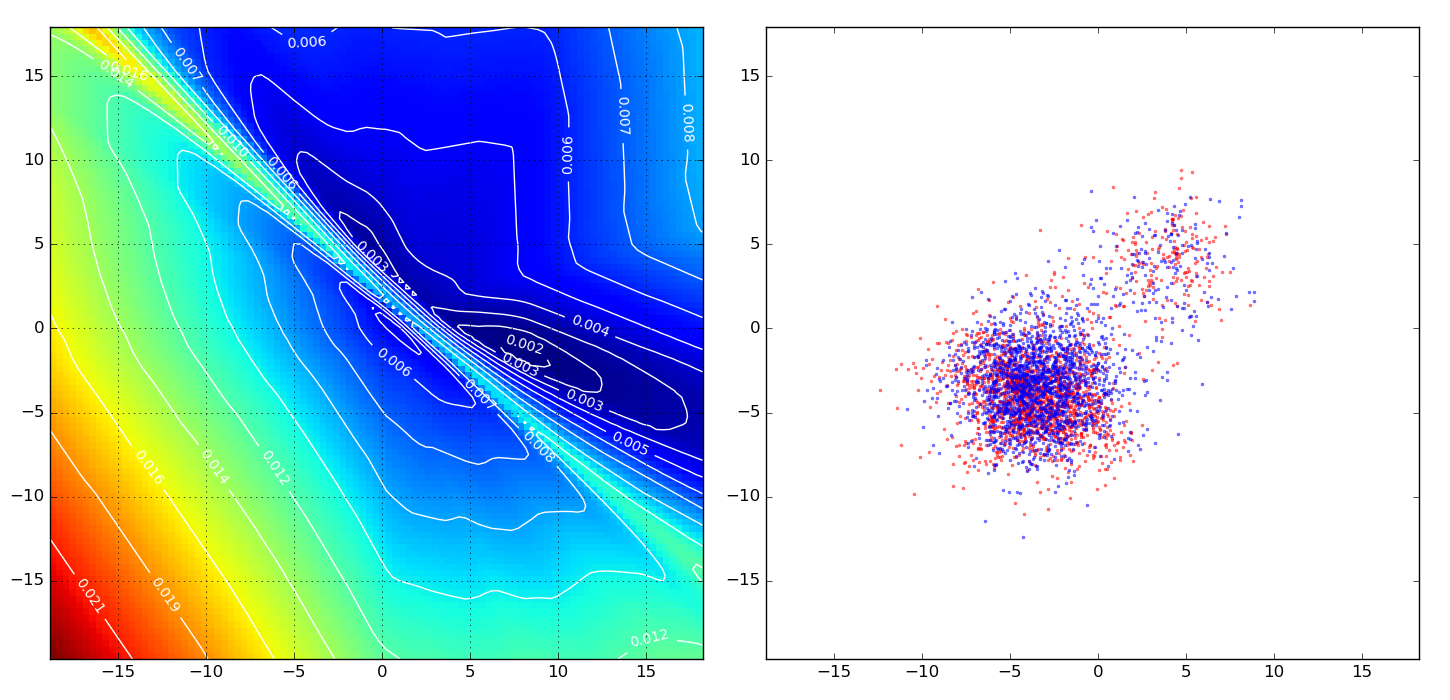}
        \caption{Energy GAN without regularization (EGAN-Const)}
	\end{subfigure}
    \caption{Learned energies and samples from baseline models whose discriminator cannot retain density information at the optimal. In the sample plots, blue dots indicate generated samples, and red dots indicate real ones.}
    \label{fig:synthetic_result_1}
\vspace{-0.6em}
\end{figure}
\begin{figure}[h]
    \begin{subfigure}{\textwidth}
		\includegraphics[width=0.32\textwidth]{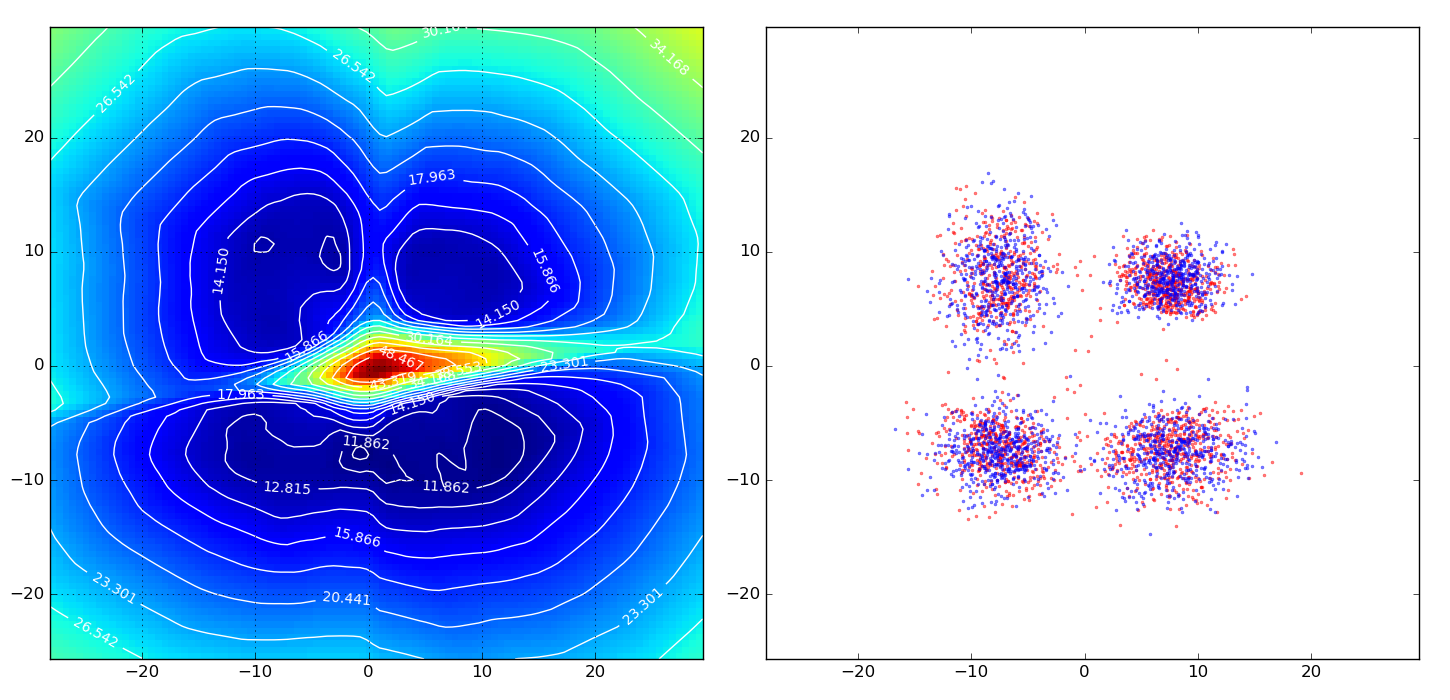}
		\includegraphics[width=0.32\textwidth]{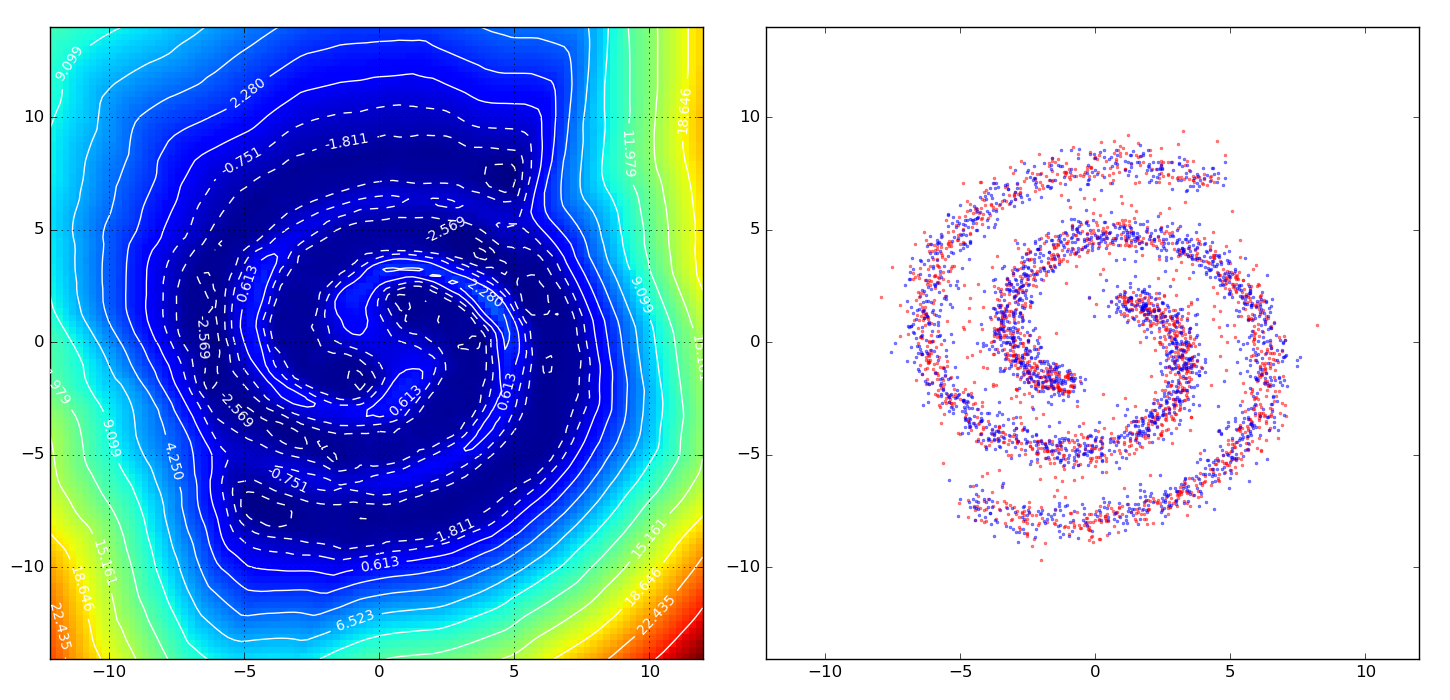}
		\includegraphics[width=0.32\textwidth]{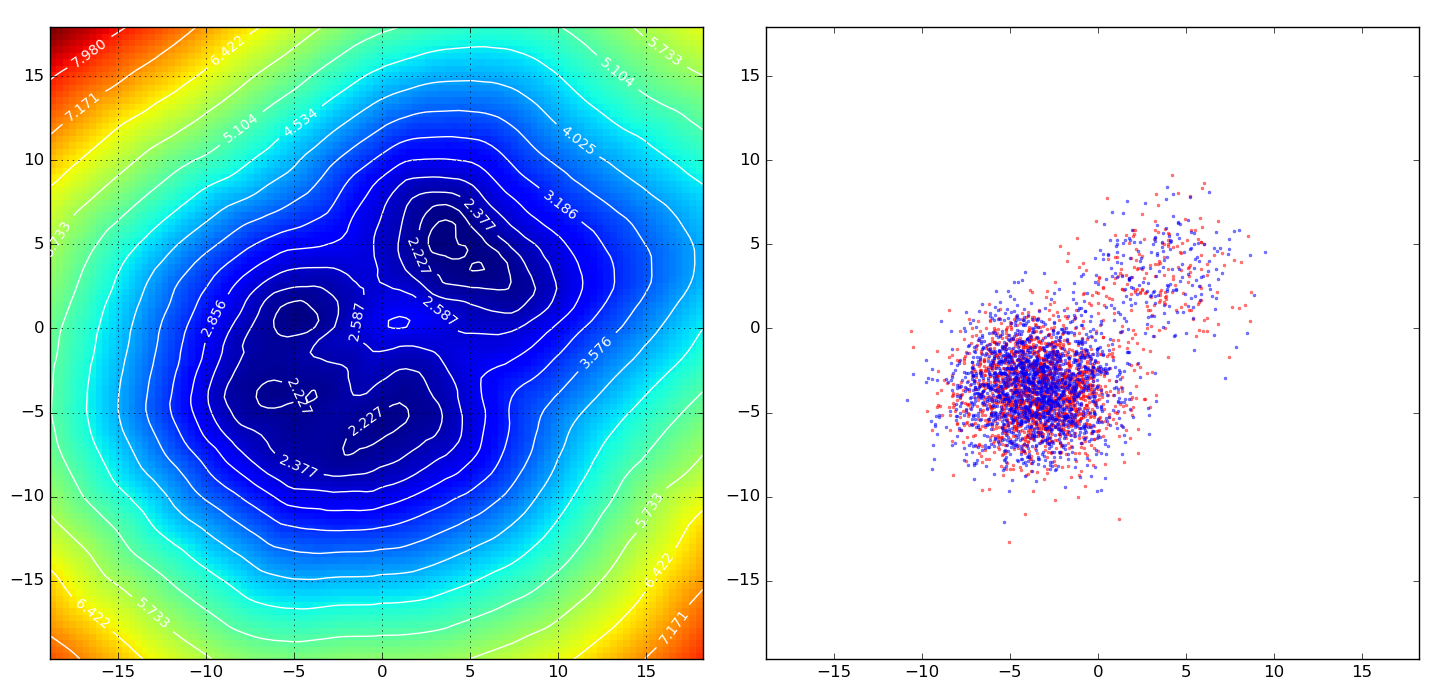}
        \caption{Entropy regularized Energy GAN with variational inference approximation (EGAN-Ent-VI)}
	\end{subfigure}
    \begin{subfigure}{\textwidth}
		\includegraphics[width=0.32\textwidth]{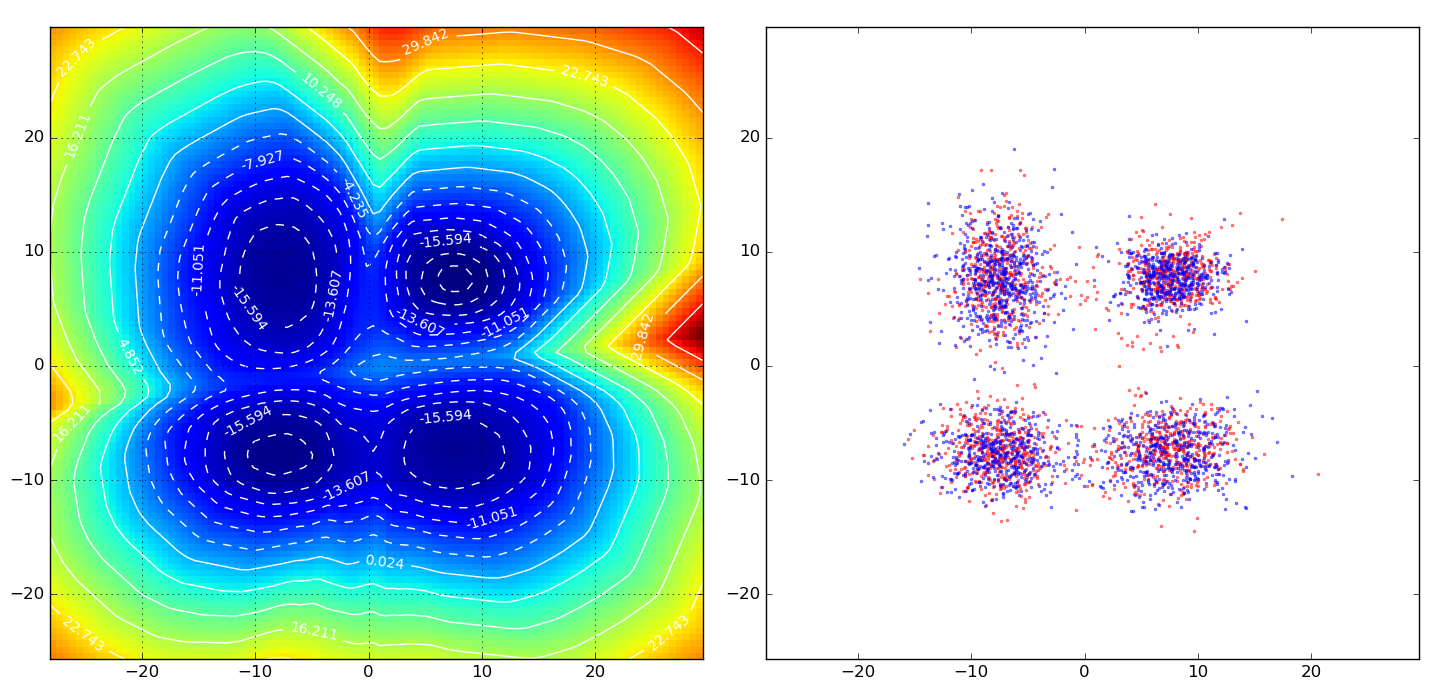}
		\includegraphics[width=0.32\textwidth]{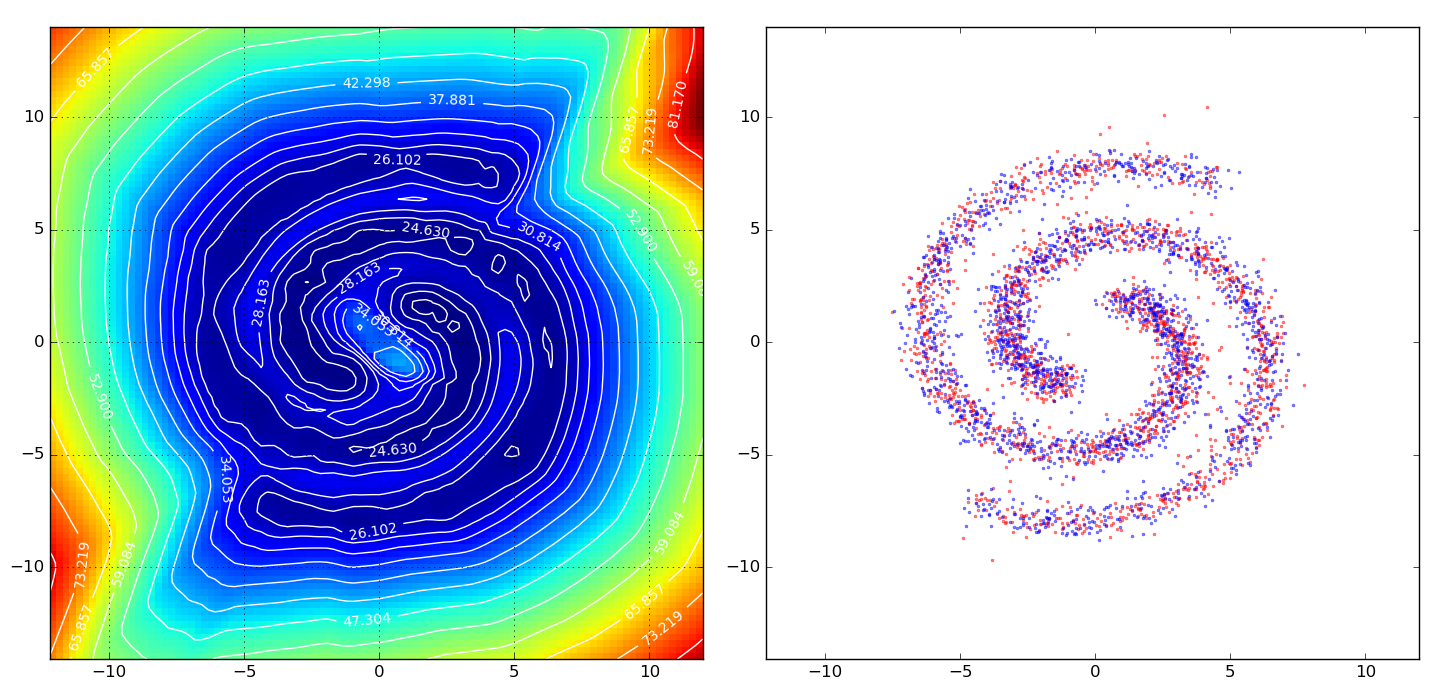}
		\includegraphics[width=0.32\textwidth]{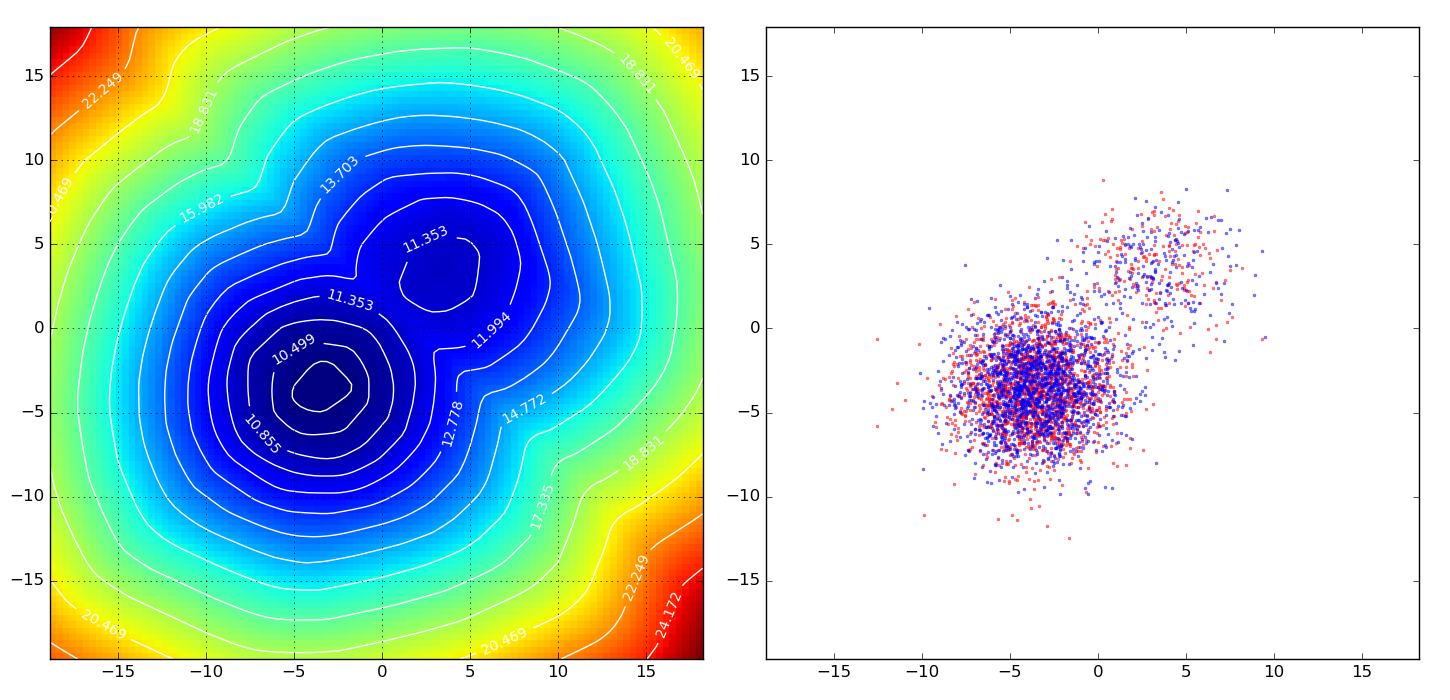}
        \caption{Entropy regularized Energy GAN with nearest neighbor approximation (EGAN-Ent-NN)}
	\end{subfigure}
    \caption{Learned energies and samples from proposed models whose discriminator can retain density information at the optimal. Blue dots are generated samples, and red dots are real ones.}
	\label{fig:synthetic_result_2}
\vspace{-0.6em}
\end{figure}

To better quantify these differences, we present detailed comparison based on KL divergence in appendix~\ref{sec:quantify_synthetic}.
What's more, the performance difference between EGAN-Ent-VI and EGAN-Ent-NN on biased Gaussian mixture reveals the limitations of the variational inference
based approximation, i.e. providing inaccurate gradients. Due to space consideratiosn, we refer interested readers to the appendix \ref{sec:vi_vs_knn} for a detailed discussion.

\subsection{Ranking NIST digits}
In this experiment, we verify that the results in synthetic datasets can translate into data with higher dimensions. 
While visualizing the learned energy function is not feasible in high-dimensional space, we can verify whether the learned energy function learns relative densities by inspecting the ranking of samples according to their assigned energies. 
We train on $28\times 28$ images of a single handwritten digit from the NIST dataset.~\footnote{\url{https://www.nist.gov/srd/nist-special-database-19}, which is an extended version of MNIST with an average of over $74$K examples per digit.} 
We compare the ability of EGAN-Ent-NN with both EGAN-Const and GAN on ranking a set of 1,000 images, half of which are generated samples and the rest are real test images. 
Figures~\ref{fig:nist_result_first} and~\ref{fig:nist_result_last} show the top-100 and bottom-100 ranked images respectively for each model, after training them on digit $1$. 
We also show in Figure~\ref{fig:mean_1} the mean of all training samples, so we can get a sense of what is the most common style (highest density) of digit 1 in NIST. 
We can notice that all of the top-ranked images by EGAN-Ent-NN look similar to the mean sample. 
In addition, the lowest-ranked images are clearly different from the mean image, with either high (clockwise or counter-clockwise) rotation degrees from the mean, or an extreme thickness level. We do not see such clear distinction in other models. We provide in the appendix \ref{sec:a-nist_experiment} the ranking of the full set of images.

\begin{figure}[t!]
  \begin{subfigure}[b]{\linewidth}
    \includegraphics[width=\linewidth]{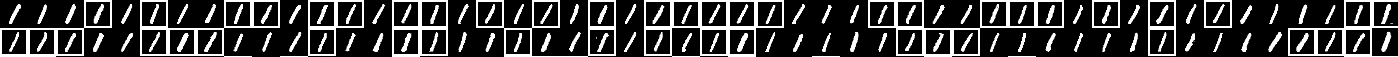}
    \caption{EGAN-Ent-NN}
    \label{fig:nist_egan_knn_first}
  \end{subfigure}
  \hfill
  \begin{subfigure}[b]{\linewidth}
    \includegraphics[width=\linewidth]{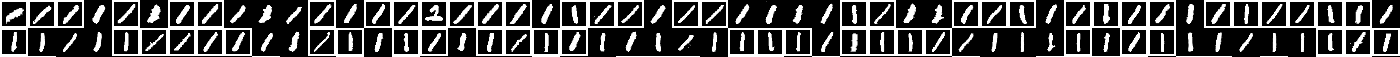}
    \caption{EGAN-Const}
    \label{fig:nist_egan_none_first}
  \end{subfigure}
  \hfill
  \begin{subfigure}[b]{\linewidth}
    \includegraphics[width=\linewidth]{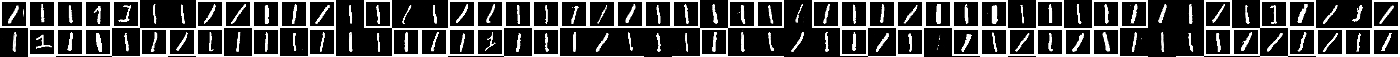}
    \caption{GAN}
    \label{fig:nist_gan_first}
  \end{subfigure}
  \caption{100 highest-ranked images out of 1000 generated and reals (bounding box) samples. }
 \label{fig:nist_result_first}
\vspace{-0.6em}
\end{figure}

\begin{figure}[t!]
  \begin{subfigure}[b]{\linewidth}
    \includegraphics[width=\linewidth]{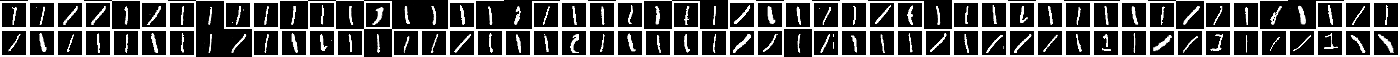}
    \caption{EGAN-Ent-NN}
    \label{fig:nist_egan_knn_last}
  \end{subfigure}
  \hfill
  \begin{subfigure}[b]{\linewidth}
    \includegraphics[width=\linewidth]{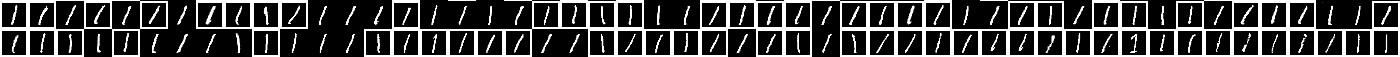}
    \caption{EGAN-Const}
    \label{fig:nist_egan_none_last}
  \end{subfigure}
  \begin{subfigure}[b]{\linewidth}
    \includegraphics[width=\linewidth]{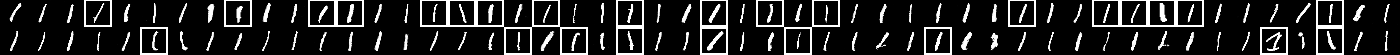}
    \caption{GAN}
    \label{fig:nist_gan_last}
  \end{subfigure}

  \caption{100 lowest-ranked images out of 1000 generated and reals (bounding box) samples. }
 \label{fig:nist_result_last}
\vspace{-0.6em}
\end{figure}

\begin{figure}[h]
  \begin{subfigure}[b]{0.45\textwidth}
    \includegraphics[width=\textwidth]{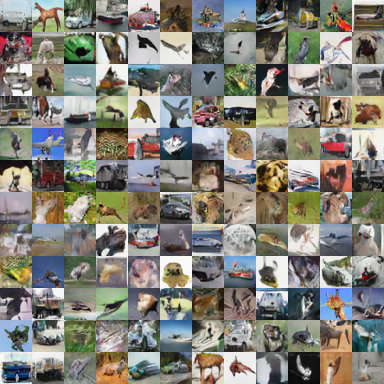}
    \caption{CIFAR-10}
    \label{fig:cifar10}
  \end{subfigure}
  \hfill
  \begin{subfigure}[b]{0.45\textwidth}
    \includegraphics[width=\textwidth]{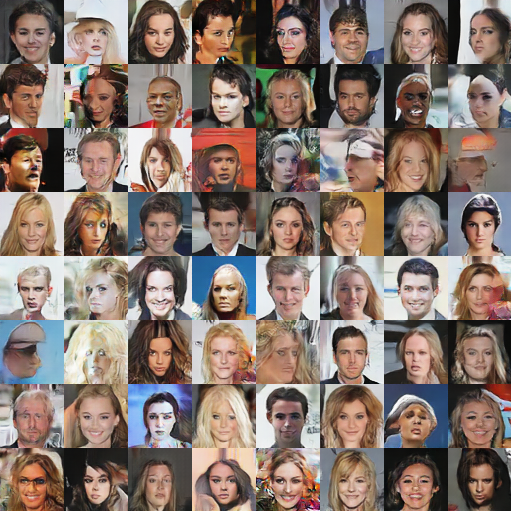}
    \caption{CelebA}
    \label{fig:celebA}
  \end{subfigure}
  \caption{Samples generated from our model.}
 \label{figure:cifar10_celebA}
\vspace{-0.6em}
\end{figure}

\subsection{Sample quality on natural image datasets}
In this last set of experiments, we evaluate the visual quality of samples generated by our model in two datasets of natural images, namely CIFAR-10 and CelebA. We employ here the variational-based approximation for entropy regularization, which can scale well to high-dimensional data. Figure~\ref{figure:cifar10_celebA} shows samples generated by EGAN-Ent-VI. We can see that despite the noisy gradients provided by the variational approximation, our model is able to generate high-quality samples.
\begin{figure}[t!]
\begin{floatrow}
\capbtabbox[9cm]{%
  \begin{tabular}{lccc}
    \toprule
    Model       & Our model &  Improved GAN$\dagger$& EGAN-Const\\
    \midrule
    Score $\pm$ std. & 7.07 $\pm$ .10  &  6.86 $\pm$ .06 & 6.7447 $\pm$ 0.09 \\
    \bottomrule
  \end{tabular}
}{
  \caption{Inception scores on CIFAR-10. $\dagger$ As reported in~\cite{salimans2016improved} without using labeled data.}
  \label{tab:cifar10-score}
}
\ffigbox{%
  \includegraphics[width=.12\linewidth]{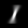}
}{%
  \caption{mean digit}%
  \label{fig:mean_1}
}
\end{floatrow}
\end{figure}

We futher validate the quality of our model's samples on CIFAR-10 using the \textit{Inception score} proposed by~\citep{salimans2016improved}~\footnote{Using the evaluation script released in \url{https://github.com/openai/improved-gan/}}. Table~\ref{tab:cifar10-score} shows the scores of our EGAN-Ent-VI, the best GAN model from~\cite{salimans2016improved} which uses only unlabeled data, and an EGAN-Const model which has the same architecture as our model. 
We notice that even without employing suggested techniques in~\cite{salimans2016improved}, energy-based models perform quite similarly to the GAN model. Furthermore, the fact that our model scores higher than EGAN-Const highlights the importance of entropy regularization in obtaining good quality samples.

\section{Conclusion}
In this paper we have addressed a fundamental limitation in adversarial learning approaches, which is their inability of providing sensible energy estimates for samples. 
We proposed a novel adversarial learning formulation which results in a discriminator function that recovers the true data energy. 
We provided a rigorous characterization of the learned discriminator in the non-parametric setting, and proposed two methods for instantiating it in the typical parametric setting. 
Our experimental results verify our theoretical analysis about the discriminator properties, and show that we can also obtain samples of state-of-the-art quality.

\section{Acknowledgements}
We would like to thank the developers of Theano~\citep{2016arXiv160502688short} for developing such a powerful tool for scientific computing. Amjad Almahairi was supported by funding from Maluuba Research.

\newpage

\bibliography{ref}
\bibliographystyle{iclr2017_conference}

\newpage

\appendix

\section{Supplementary materials for Section \ref{sec:formulation}}
\subsection{Optimal discriminator form under the proposed formulation}
\label{sec:proof_for_formulation}

\begin{proof}[Proof of proposition \ref{thm:general_optimal_disc}]
Refining the Lagrange $L(\pg, c)$ by introducing additional dual variables for the probability constraints (the second and third), the new Lagrange function has the form
\begin{equation}
\label{eq:lagrange_form}
L(\pg, c, \mu, \lambda) = K(\pg) + \sum_{x \in \mathcal{X}} c(x) \Big( \pg(x) - \pd(x) \Big) - \sum_{x \in \mathcal{X}} \mu(x)\pg(x)  + \lambda (\sum_{x \in \mathcal{X}} \pg(x) - 1)
\end{equation}
where $c(x) \in \R, \forall x$, $\mu(x) \in \R_{+},\forall x$, and $\lambda \in \R$ are the dual variables.
The KKT conditions for the optimal primal and dual variables are as follows
\begin{equation}
\begin{aligned}
	\frac{\partial K(\pg)}{\partial \pg(x)}\bigg|_{\pg=\pd} + c^*(x) - \mu^*(x) + \lambda^* = 0,& \quad\forall x & \text{(stationarity)} \\
	\mu^*(x) \pg^*(x) = 0, &\quad\forall x & \text{(complement slackness)} \\
    \mu^*(x) \geq 0, &\quad\forall x & \text{(dual feasibility)} \\
	\pg^*(x) \geq 0, \quad \pg^*(x) = \pd(x), &\quad\forall x & \text{(primal feasibility)} \\
	\sum_{x \in \mathcal{X}} \pg(x) = 1 && \text{(primal feasibility)}
\end{aligned}
\end{equation}
Rearranging the conditions above, we get $\pg^*(x) = \pd(x), \forall x \in \mathcal{X}$ as well as equation \eqref{eq:optimal_disc}, which concludes the proof.
\end{proof}

\subsection{Optimal conditions of EBGAN}
\label{sec:ebgan_proof}

In \citep{zhao2016energy}, the training objectives of the generator and the discriminator cannot be written as a single minimax optimization problem since the margin structure is only applied to the objective of the discriminator.
In addition, the discriminator is designed to produce the mean squared reconstruction error of an auto-encoder structure.
This restricted the range of the discriminator output to be non-negative, which is equivalent to posing a set constraint on the discriminator under the non-parametric setting.

Thus, to characterize the optimal generator and discriminator, we adapt the same analyzing logic used in the proof sketch of the original GAN~\citep{goodfellow2014generative}.
Specifically, given a specific generator distribution $\pg$, the optimal discriminator function given the generator distribution $c^*(x; \pg)$ can be derived by examining the objective of the discriminator.
Then, the conditional optimal discriminator function is substituted into the training objective of $\pg$, simplifying the ``adversarial'' training as a minimizing problem only w.r.t. $\pg$, which can be well analyzed.

Firstly, given any generator distribution $\pg$, the EBGAN training objective for the discriminator can be written as the following form
\begin{equation}
\label{eq:ebgan_objective_disc}
\begin{aligned}
c^*(x; \pg) 
&= \argmax_{c \in \mathcal{C}} \quad -\E_{\pg}\max(0, m - c(x)) - \E_{\pd} c(x) \\
&= \argmax_{c \in \mathcal{C}} \quad \E_{\pg}\min(0, c(x) - m) - \E_{\pd} c(x)
\end{aligned}
\end{equation}
where $\mathcal{C} = \{c: c(x) \geq 0, \forall x \in \mathcal{X}\}$ is the set of allowed non-negative discriminator functions.
Note this set constraint comes from the fact the mean squared reconstruction error as discussed above.

Since the problem \eqref{eq:ebgan_objective_disc} is independent w.r.t. each $x$, the optimal solution can be easily derived as
\begin{equation}
\label{eq:ebgan_optimal_c}
c^*(x; \pg) =
\begin{cases}
0,& \pg(x) < \pd(x) \\
m,& \pg(x) > \pd(x) \\
\alpha_x, & \pg(x) = \pd(x) > 0 \\
\beta_x, & \pg(x) = \pd(x) = 0
\end{cases}
\end{equation}
where $\alpha_x \in [0, m]$ is an under-determined number, a $\beta_x \in [0, \infty)$ is another under-determined non-negative real number, and the subscripts in $m, \alpha_x, \beta_x$ reflect that fact that these under-determined values can be distinct for different $x$.

This way, the overall training objective can be cast into a minimization problem w.r.t. $\pg$,
\begin{equation}
\label{eq:ebgan_objective_generator}
\begin{aligned}
\pg^* &= \argmin_{\pg \in \mathcal{P}} \, \E_{x \sim \pg} c^*(x; \pg) - \E_{x \sim \pd} c^*(x; \pg) \\
&= \argmin_{\pg \in \mathcal{P}} \, \sum_{x \in \mathcal{X}} \Big[\pg(x) - \pd(x)\Big] c^*(x; \pg)
\end{aligned}
\end{equation}
where the second term of the first line is implicitly defined as the problem is an adversarial game between $\pg$ and $c$.

\begin{proposition}
The global optimal of the EBGAN training objective is achieved if and only if $\pg = \pd$. At that point, $c^*(x)$ is fully under-determined.
\end{proposition}

\begin{proof}

The proof is established by showing contradiction. 

Firstly, assume the optimal $\pg^* \neq \pd$. 
Thus, there must exist a non-equal set $\mathcal{X}_{\neq} = \{x \mid \pd(x) \neq \pg^*(x) \}$, which can be further splitted into two subsets, 
the greater-than set $\mathcal{X}_{>} = \{x \mid \pg^*(x) > \pd(x)\}$,
and the less-than set $\mathcal{X}_{<} = \{x \mid \pg^*(x) < \pd(x)\}$.
Similarly, we define the equal set $\mathcal{X}_{=} = \{x : \pg^*(x) = \pd(x)\}$.
Obviously, $\mathcal{X}_{>} \bigcup \mathcal{X}_{<} \bigcup \mathcal{X}_{=} = \mathcal{X}$.

Let $L(\pg) = \sum_{x \in \mathcal{X}} \Big[\pg(x) - \pd(x)\Big] c^*(x; \pg)$, substituting the results from equation \eqref{eq:ebgan_optimal_c} into \eqref{eq:ebgan_objective_generator}, the $L(\pg)^*$ can be written as
\begin{equation}
\label{eq:val_g}
\begin{aligned}
L(\pg^*)
&= \sum_{x \in \mathcal{X}_{<} \bigcup \mathcal{X}_{<} \bigcup \mathcal{X}_{=}} \big[ \pg^*(x) - \pd(x) \big] c^*(x; \pg^*) \\
&= \sum_{x \in \mathcal{X}_{<}} \big[ \pg^*(x) - \pd(x) \big] c^*(x; \pg^*) + \sum_{x \in \mathcal{X}_{>}} \big[ \pg^*(x) - \pd(x) \big] c^*(x; \pg^*) \\
&= m\sum_{x \in \mathcal{X}_{>}} \pg^*(x) - \pd(x) \\
&> 0
\end{aligned}
\end{equation}

However, when $\pg^\prime = \pd$, we have
\begin{equation}
\label{eq:val_g_opt}
L(\pg^\prime) = 0 < L(\pg^*)
\end{equation}
which contradicts the optimal (miminum) assumption of $\pg^*$.
Hence, the contradiction concludes that at the global optimal, $\pg^* = \pd$.
By equation \eqref{eq:ebgan_optimal_c}, it directly follows that $c^*(x; \pg^*) = \alpha_x$,
which completes the proof.
\end{proof}

\subsection{Analysis of adding additional training signal to GAN formulation}
\label{sec:fgan_reg}
To show that simply adding the same training signal to GAN will not lead to the same result, it is more convenient to directly work with the formulation of $f$-GAN~\citep[equation (6)]{nowozin2016f} family, which include the original GAN formulation as a special case.

Specifically, the general $f$-GAN formulation takes the following form
\begin{equation}
\label{eq:fgan_formulation}
\max_{c} \min_{\pg \in \mathcal{P}} \quad
	\E_{x \sim \pg} \big[ f^\star(c(x)) \big] - 
    \E_{x \sim \pd} \big[ c(x) \big],
\end{equation}
where the $f^\star(\cdot)$ denotes the convex conjugate~\citep{boyd2004convex} of the $f$-divergence function.
The optimal condition of the discriminator can be found by taking the variation w.r.t. $c$, which gives the optimal discriminator 
\begin{equation}
\label{eq:optimal_fgan_disc}
c^*(x) = f^\prime(\frac{\pd(x)}{\pg(x)})
\end{equation}
where $f^\prime(\cdot)$ is the first-order derivative of $f(\cdot)$.
Note that, even when we add an extra term $L(\pg)$ to equation \eqref{eq:fgan_formulation}, since the term $K(\pg)$ is a constant w.r.t. the discriminator, it does not change the result given by equation \eqref{eq:optimal_fgan_disc} about the optimal discriminator. 
As a consequence, for the optimal discriminator to retain the density information, it effectively means $\pg \neq \pd$.
Hence, there will be a contradiction if both $c^*(x)$ retains the density information, and the generator matches the data distribution.

Intuitively, this problem roots in the fact that $f$-divergence is quite ``rigid'' in the sense that given the $\pg(x)$ it only allows one fixed point for the discriminator.
In comparison, the divergence used in our proposed formulation, which is the expected cost gap, is much more flexible.
By the expected cost gap itself, i.e. without the $K(\pg)$ term, the optimal discriminator is actually under-determined.

\section{Supplementary Materials for section \ref{sec:experiments}}

\subsection{Experiment setting}
\label{sec:experiment_setting}
Here, we specify the neural architectures used for experiements presented in Section \ref{sec:experiments}.

Firstly, for the Egan-Ent-VI model, we parameterize the approximate posterior distribution $\qg(z \mid x)$ with a diagonal Gaussian distribution, whose mean and covariance matrix are the output of a trainable inference network, i.e.
\begin{equation}
\begin{aligned}
\qg(z \mid x) &= \mathcal{N}(\mu, \mathbf{I}\sigma^2) \\
\mu, \log \sigma &= f^\text{infer}(x)
\end{aligned}
\end{equation}
where $f^\text{infer}$ denotes the inference network, and $\mathbf{I}$ is the identity matrix.
Note that the Inference Network only appears in the Egan-Ent-VI model.

For experiments with the synthetic datasets, the following fully-connected feed forward neural networks are employed 
{\footnotesize
\begin{itemize}[leftmargin=16pt,labelindent=16pt]
\item Generator: \texttt{FC(4,128)-BN-ReLU-FC(128,128)-BN-ReLU-FC(128,2)}
\item Discriminator: \texttt{FC(2,128)-ReLU-FC(128,128)-ReLU-FC(128,1)}
\item Inference Net: \texttt{FC(2,128)-ReLU-FC(128,128)-ReLU-FC(128,4*2)}
\end{itemize}}
where \texttt{FC} and \texttt{BN} denote fully-connected layer and batch normalization layer respectively.
Note that since the input noise to the generator has dimension \texttt{4}, the Inference Net output has dimension \texttt{4*2}, where the first 4 elements correspond the inferred mean, and the last 4 elements correspond to the inferred diagonal covariance matrix in log scale.

For the handwritten digit experiment, we closely follow the DCGAN~\citep{radford2015unsupervised} architecture with the following configuration
{\footnotesize
\begin{itemize}[leftmargin=16pt,labelindent=16pt]
\item Generator: \texttt{FC(10,512*7*7)-BN-ReLU-DC(512,256;4c2s)-BN-ReLU}\\ 
\texttt{-DC(256,128;4c2s)-BN-ReLU-DC(128,1;3c1s)-Sigmoid}
\item Discriminator: \texttt{CV(1,64;3c1s)-BN-LRec-CV(64,128;4c2s)-BN-LRec}\\ 
\texttt{-CV(128,256;4c2s)-BN-LRec-FC(256*7*7,1)}
\item Inference Net: \texttt{CV(1,64;3c1s)-BN-LRec-CV(64,128;4c2s)-BN-LRec}\\ 
\texttt{-CV(128,256;4c2s)-BN-LRec-FC(256*7*7,10*2)}
\end{itemize}}
Here, \texttt{LRec} is the leaky rectified non-linearity recommended by \citet{radford2015unsupervised}. 
In addition, \texttt{CV(128,256,4c2s)} denotes a convolutional layer with 128 input channels, 256 output channels, and kernel size 4 with stride 2. Similarly, \texttt{DC(256,128,4c2s)} denotes a corresponding transposed convolutional operation. 
Compared to the original DCGAN architecture, the discriminator under our formulation does not have the last sigmoid layer which squashes a scalar value into a probability in [0, 1]. 

For celebA experiment with $64 \times 64$ color images, we use the following architecture 
{\footnotesize
\begin{itemize}[leftmargin=16pt,labelindent=16pt]
\item Generator: \texttt{FC(10,512*4*4)-BN-ReLU-DC(512,256;4c2s)-BN-ReLU-DC(256,128;4c2s)}\\ 
\texttt{-BN-ReLU-DC(256,128;4c2s)-BN-ReLU-DC(128,3;4c2s)-Tanh}
\item Discriminator: \texttt{CV(3,64;4c2s)-BN-LRec-CV(64,128;4c2s)-BN-LRec-CV(128,256;4c2s)}\\ 
\texttt{-BN-LRec-CV(256,256;4c2s)-BN-LRec-FC(256*4*4,1)}
\item Inference Net: \texttt{CV(3,64;4c2s)-BN-LRec-CV(64,128;4c2s)-BN-LRec-CV(128,256;4c2s)}\\ 
\texttt{-BN-LRec-CV(256,256;4c2s)-BN-LRec-FC(256*4*4,10*2)}
\end{itemize}}
For Cifar10 experiment, where the image size is $32 \times 32$, similar architecture is used
{\footnotesize
\begin{itemize}[leftmargin=16pt,labelindent=16pt]
\item Generator: \texttt{FC(10,512*4*4)-BN-ReLU-DC(512,256;4c2s)-BN-ReLU-DC(256,128;3c1s)}\\ 
\texttt{-BN-ReLU-DC(256,128;4c2s)-BN-ReLU-DC(128,3;4c2s)-Tanh}
\item Discriminator: \texttt{CV(3,64;3c1s)-BN-LRec-CV(64,128;4c2s)-BN-LRec-CV(128,256;4c2s)}\\ 
\texttt{-BN-LRec-CV(256,256;4c2s)-BN-LRec-FC(256*4*4,1)}
\item Inference Net: \texttt{CV(3,64;3c1s)-BN-LRec-CV(64,128;4c2s)-BN-LRec-CV(128,256;4c2s)}\\ 
\texttt{-BN-LRec-CV(256,256;4c2s)-BN-LRec-FC(256*4*4,10*2)}
\end{itemize}}

Given the chosen architectures, we follow \citet{radford2015unsupervised} and use Adam as the optimization algorithm.
For more detailed hyper-parameters, please refer to the code.

\subsection{Quantitative comparison of different models}
\label{sec:quantify_synthetic}
\setlength\tabcolsep{3pt}
\begin{table}[h]
\scriptsize
\centering
\begin{tabular}{l|cccc|cccc|cc}
\toprule
& \multicolumn{10}{c}{Gaussian Mixture: \quad $\text{KL}(\pd \| p_\text{emp}) = 0.0291$, \quad $\text{KL}(p_\text{emp} \| \pd) = 0.0159$} \\
\midrule
KL Divergence & $\pg \| p_\text{emp}$ & $p_\text{emp} \| \pg$ & $\pg \| \pd$ & $\pd \| \pg$ & $p_\text{disc} \| p_\text{emp}$ & $p_\text{emp} \| p_\text{disc}$ & $p_\text{disc} \| \pd$ & $\pd \| p_\text{disc}$ & $\pg \| p_\text{disc}$ & $p_\text{disc} \| \pg$ \\
\midrule
GAN         & 0.3034 & 0.5024 & 0.2498 & 0.4807 & 6.7587 & 2.0648 & 6.2020 & 2.0553 & 2.4596 & 7.0895 \\
EGAN-Const  & 0.2711 & 0.4888 & 0.2239 & 0.4735 & 6.7916 & 2.1243 & 6.2159 & 2.1149 & 2.5062 & 7.0553 \\
EGAN-Ent-VI & 0.1422 & 0.1367 & 0.0896 & 0.1214 & 0.8866 & 0.6532 & 0.7215 & 0.6442 & 0.7711 & 1.0638 \\
EGAN-Ent-NN & \textbf{0.1131} & \textbf{0.1006} & \textbf{0.0621} & \textbf{0.0862} & \textbf{0.0993} & \textbf{0.1356} & \textbf{0.0901}   & \textbf{0.1187} & \textbf{0.1905} & \textbf{0.1208} \\
\bottomrule \toprule
& \multicolumn{10}{c}{Biased Gaussian Mixture: \quad $\text{KL}(\pd \| p_\text{emp}) = 0.0273$, \quad $\text{KL}(p_\text{emp} \| \pd) = 0.0144$} \\
\midrule
KL Divergence & $\pg \| p_\text{emp}$ & $p_\text{emp} \| \pg$ & $\pg \| \pd$ & $\pd \| \pg$ & $p_\text{disc} \| p_\text{emp}$ & $p_\text{emp} \| p_\text{disc}$ & $p_\text{disc} \| \pd$ & $\pd \| p_\text{disc}$ & $\pg \| p_\text{disc}$ & $p_\text{disc} \| \pg$ \\
\midrule
GAN         & 0.0788 & 0.0705 & 0.0413 & 0.0547 & 7.1539 & 2.5230 & 6.4927 & 2.5018 & 2.5205 & 7.1140 \\
EGAN-Const  & 0.1545 & 0.1649 & 0.1211 & 0.1519 & 7.1568 & 2.5269 & 6.4969 & 2.5057 & 2.5860 & 7.1995 \\
EGAN-Ent-VI & \textbf{0.0576} & 0.0668 & \textbf{0.0303} & 0.0518 & 3.9151 & 1.3574 & 2.9894 & 1.3365 & 1.4052 & 4.0632 \\
EGAN-Ent-NN & 0.0784 & \textbf{0.0574} & 0.0334 & \textbf{0.0422} & \textbf{0.8505} & \textbf{0.3480} & \textbf{0.5199} & \textbf{0.3299} & \textbf{0.3250} & \textbf{0.7835} \\
\bottomrule\toprule
&\multicolumn{10}{c}{Two-spiral Gaussian Mixture: \quad $\text{KL}(\pd \| p_\text{emp}) = 0.3892$, \quad $\text{KL}(p_\text{emp} \| \pd) = 1.2349$} \\
\midrule
KL Divergence & $\pg \| p_\text{emp}$ & $p_\text{emp} \| \pg$ & $\pg \| \pd$ & $\pd \| \pg$ & $p_\text{disc} \| p_\text{emp}$ & $p_\text{emp} \| p_\text{disc}$ & $p_\text{disc} \| \pd$ & $\pd \| p_\text{disc}$ & $\pg \| p_\text{disc}$ & $p_\text{disc} \| \pg$ \\
\midrule
GAN         & 0.5297 & 0.2701 & 0.3758 & 0.7240 & 6.3507 & 1.7180 & 4.3818 & 1.0866 & 1.6519 & 5.7694 \\
EGAN-Const  & 0.7473 & 1.0325 & 0.7152 & 1.6703 & 5.9930 & 1.5732 & 3.9749 & 0.9703 & 1.8380 & 6.0471 \\
EGAN-Ent-VI & 0.2014 & 0.1260 & \textbf{0.4283} & \textbf{0.8399} & 1.1099 & 0.3508 & \textbf{0.3061} & \textbf{0.4037} & 0.4324 & 0.9917 \\
EGAN-Ent-NN &\textbf{ 0.1246} & \textbf{0.1147} & 0.4475 & 1.2435 & \textbf{0.1036} & \textbf{0.0857} & 0.4086 & 0.7917 & \textbf{0.1365} & \textbf{0.1686} \\
\bottomrule
\end{tabular}
\caption{Pairwise KL divergence between distributions. Bold face indicate the lowest divergence within group.}
\label{tab:kl}
\end{table}
\setlength\tabcolsep{6pt}

In order to quantify the quality of recovered distributions, we compute the pairwise KL divergence of the following four distributions:
\begin{itemize}
\item The real data distribution with analytic form, denoted as $\pd$
\item The empirical data distribution approximated from the 100K training data, denoted as $p_\text{emp}$
\item The generator distribution approximated from 100K generated data, denoted as $\pg$
\item The discriminator distribution re-normalized from the learned energy, denoted as $p_\text{disc}$
\end{itemize}
Since the synthetic datasets are two dimensional, we approximate both the empirical data distribution and the generator distribution using the simple histogram estimation.
Specifically, we divide the canvas into a 100-by-100 grid, and assign each sample into its nearest grid cell based on euclidean distance.
Then, we normalize the number of samples in each cell into a proper distribution.
When recovering the discriminator distribution from the learned energy, we assume that $\mu^*(x) = 0$ (i.e. infinite data support), and discretize the distribution into the same grid cells
\begin{equation*}
p_\text{disc}(x) = \frac{\exp(-c^*(x))}{\sum_{x^\prime \in \text{Grid}} \exp(-c^*(x^\prime))}, \forall x \in \text{Grid}
\end{equation*}
Based on these approximation, Table \ref{tab:kl} summarizes the results.
For all measures related to the discriminator distribution, EGAN-Ent-VI and EGAN-Ent-NN significantly outperform the other two baseline models, which matches our visual assessment in Figure \ref{fig:synthetic_result_1} and \ref{fig:synthetic_result_2}.
Meanwhile, the generator distributions learned from our proposed framework also achieve relatively lower divergence to both the empirical data distribution and the true data distribution.

\subsection{Comparison of the entropy (gradient) approximation methods}
\label{sec:vi_vs_knn}

In order to understand the performance difference between EGAN-Ent-VI and EGAN-Ent-NN, we analyze the quality of the entropy gradient approximation during training.
To do that, we visualize some detailed training information in Figure \ref{fig:entropy_gradient_analysis}.
\begin{figure}[h]
	\centering
	\begin{subfigure}{\textwidth}
		\includegraphics[width=\textwidth]{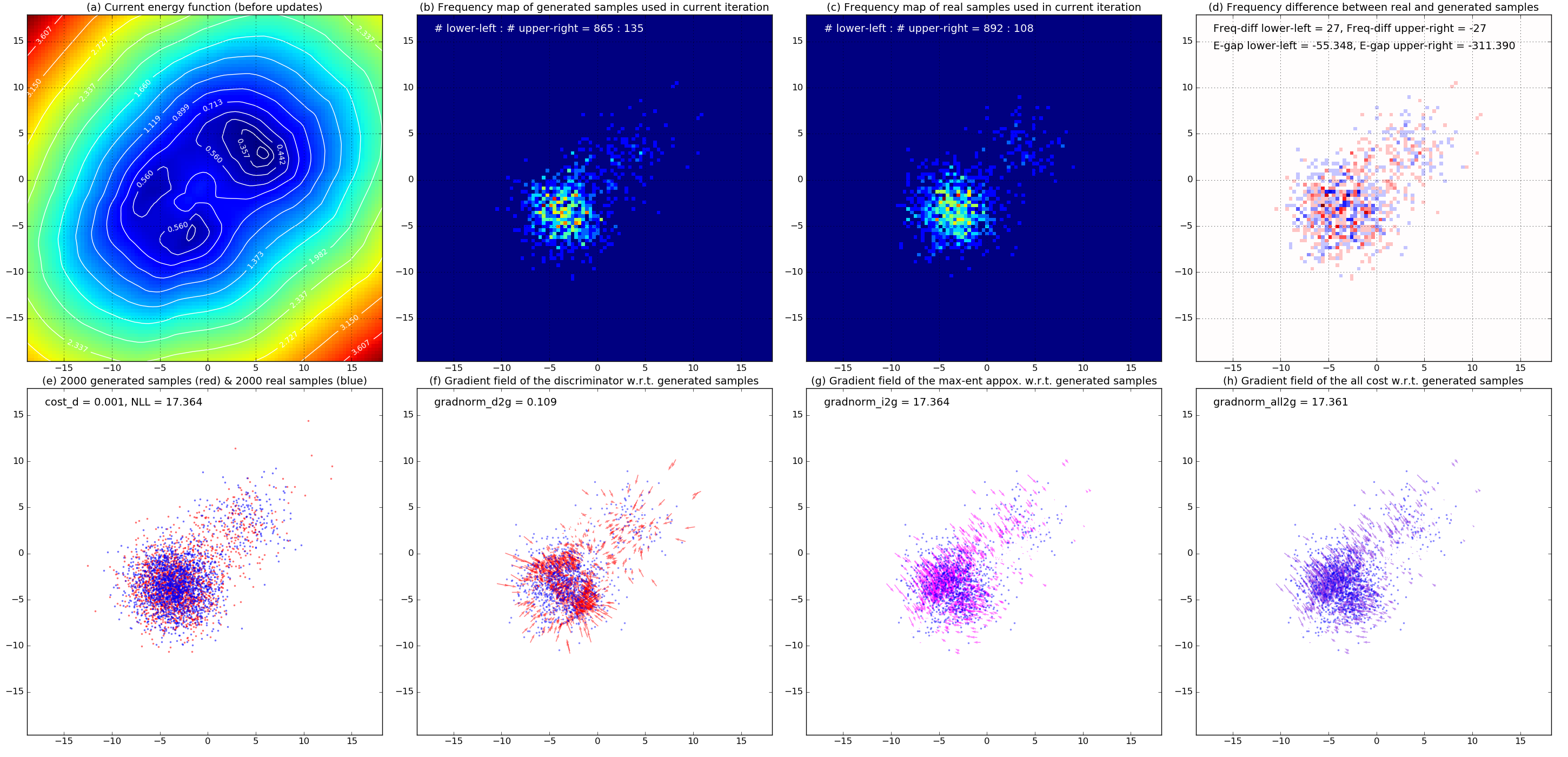}
        \caption{Training details under variational inference entropy approximation}
        \label{fig:entropy_gradient_vi}
	\end{subfigure}
    \begin{subfigure}{\textwidth}
		\includegraphics[width=\textwidth]{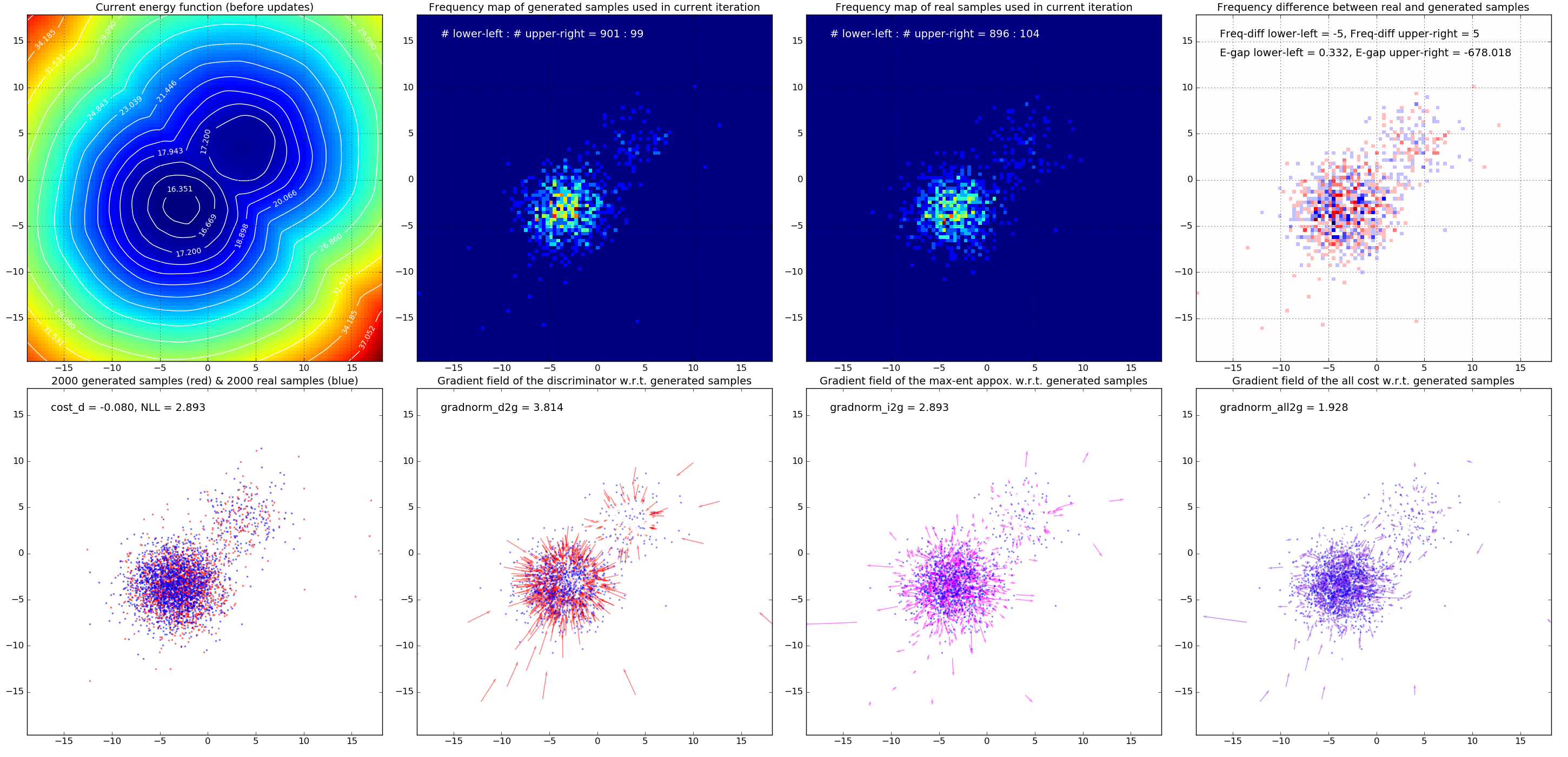}
        \caption{Training details under nearest neighbor entropy approximation}
        \label{fig:entropy_gradient_knn}
	\end{subfigure}
	\caption{For convenience, we will use Fig. (i,j) to refer to the subplot in row i, column j. Fig. (1,1): current energy plot. Fig. (1,2): frequency map of generated samples in the current batch. Fig. (1,3): frequency map of real samples in the current batch. Fig-(1,4): frequency difference between real and generated samples. Fig. (2,1) comparison between more generated from current model and real sample. Fig. (2,2): the discriminator gradient w.r.t. each training sample. Fig. (2,3): the entropy gradient w.r.t. each training samples. Fig. (2,4): all gradient (discriminator + entropy) w.r.t. each training sample.}
	\label{fig:entropy_gradient_analysis}
\end{figure}

As we can see in figure \ref{fig:entropy_gradient_vi}, the viarational entropy gradient approximation w.r.t. samples is not accurate:
\begin{itemize}
\item It is inaccurate in terms of gradient direction. 
Ideally, the direction of the entropy gradient should be pointing from the center of its closest mode towards the surroundings, with the direction orthogonal to the implicit contour in Fig. (1,2).
However, the direction of gradients in the Fig. (2,3) does not match this. 

\item It is inaccurate in magnitude. 
As we can see, the entropy approximation gradient (Fig. (2,3)) has much larger norm than the discriminator gradient (Fig. (2,2)). 
As a result, the total gradient (Fig. (2,4)) is fully dominated by the entropy approximation gradient.
Thus, it usually takes much longer for the generator to learn to generate rare samples, and the training also proceeds much slower compared to the nearest neighbor based approximation.
\end{itemize}

In comparison, the nearest neighbor based gradient approximation is much more accurate as shown in \ref{fig:entropy_gradient_knn}.
As a result, it leads to more accurate energy contour, as well as faster training.
What's more, from Figure \ref{fig:entropy_gradient_knn} Fig. (2,4), we can see the entropy gradient does have the cancel-out effect on the discriminator gradient, which again matches our theory.

\subsection{Ranking NIST Digits}
\label{sec:a-nist_experiment}
Figure~\ref{fig:nist_result_all} shows the ranking of all 1000 generated and real images (from the test set) for three models: EGAN-Ent-NN, EGAN-Const, and GAN. We can clearly notice that in EGAN-Ent-NN the top-ranked digits look very similar to the mean digit. From the upper-left corner to the lower-right corner, the transition trend is: the rotation degree increases, and the digits become increasingly thick or thin compared to the mean. 
In addition, samples in the last few rows do diverge away from the mean image: either highly diagonal to the right or left, or have different shape: very thin or thick, or typewriter script. Other models are not able to achieve a similar clear distinction for high versus low probability images.
Finally, we consistently observe the same trend in modeling other digits, which are not shown in this paper due to space constraint.

\begin{figure}[]
  \begin{subfigure}[b]{\linewidth}
    \includegraphics[width=\linewidth]{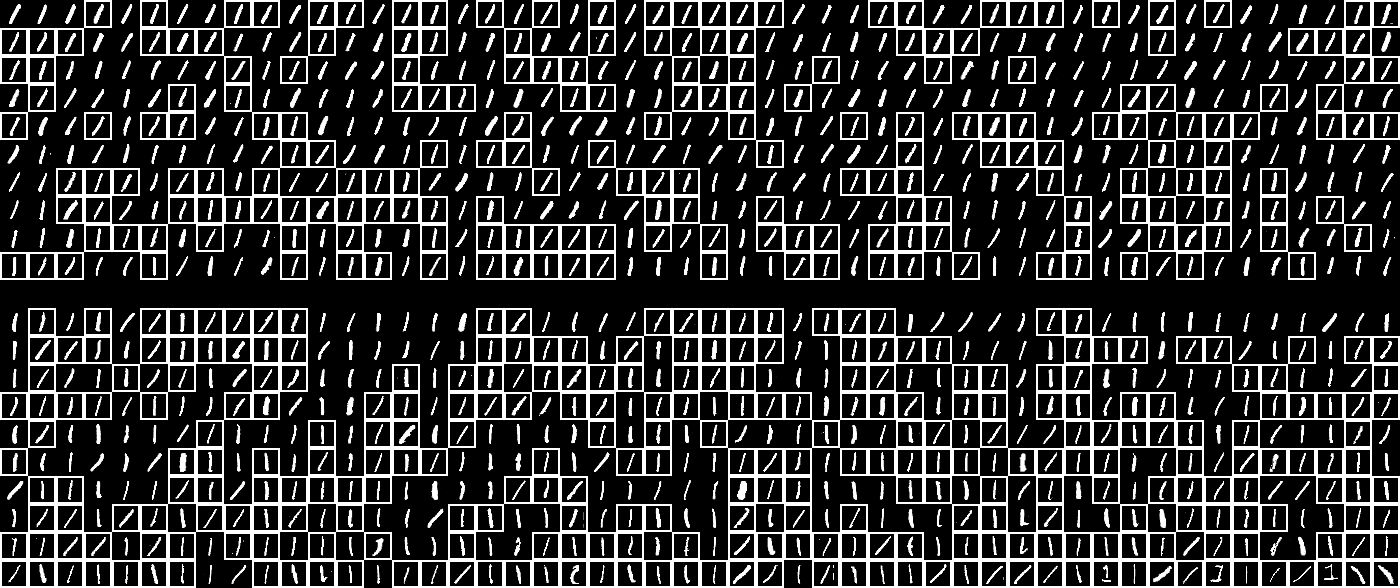}
    \caption{EGAN-Ent-NN}
    \label{fig:nist_egan_knn_all}
  \end{subfigure}
  \hfill
  \begin{subfigure}[b]{\linewidth}
    \includegraphics[width=\linewidth]{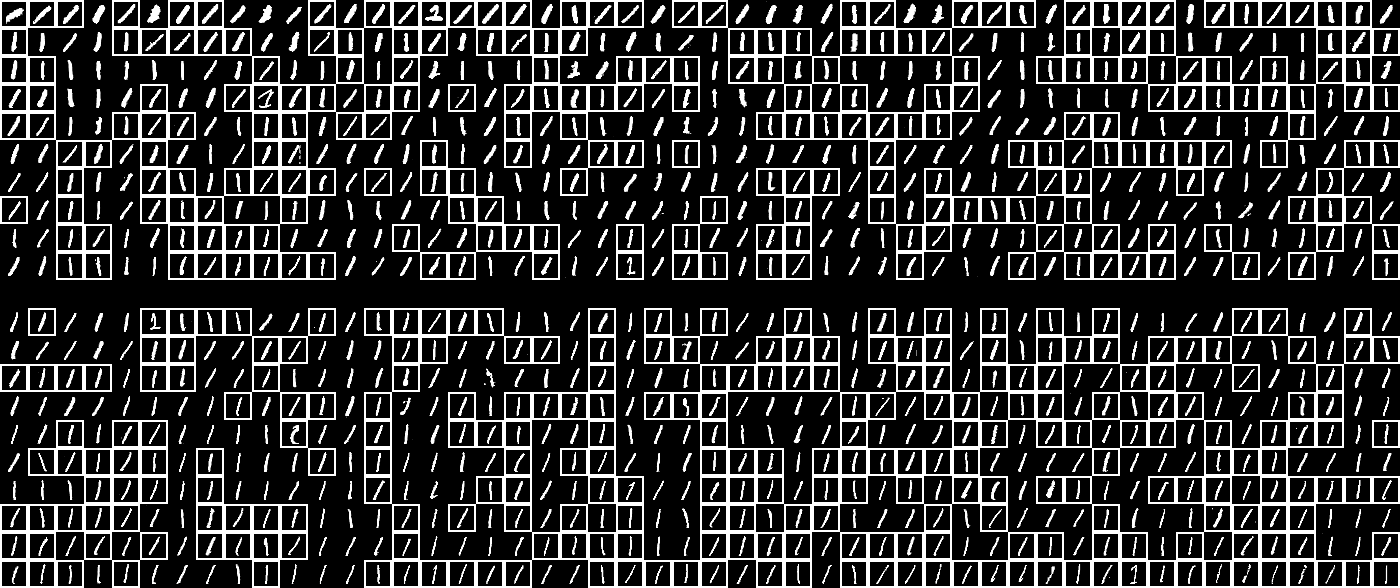}
    \caption{EGAN-Const}
    \label{fig:nist_egan_none_all}
  \end{subfigure}
  \hfill
  \begin{subfigure}[b]{\linewidth}
    \includegraphics[width=\linewidth]{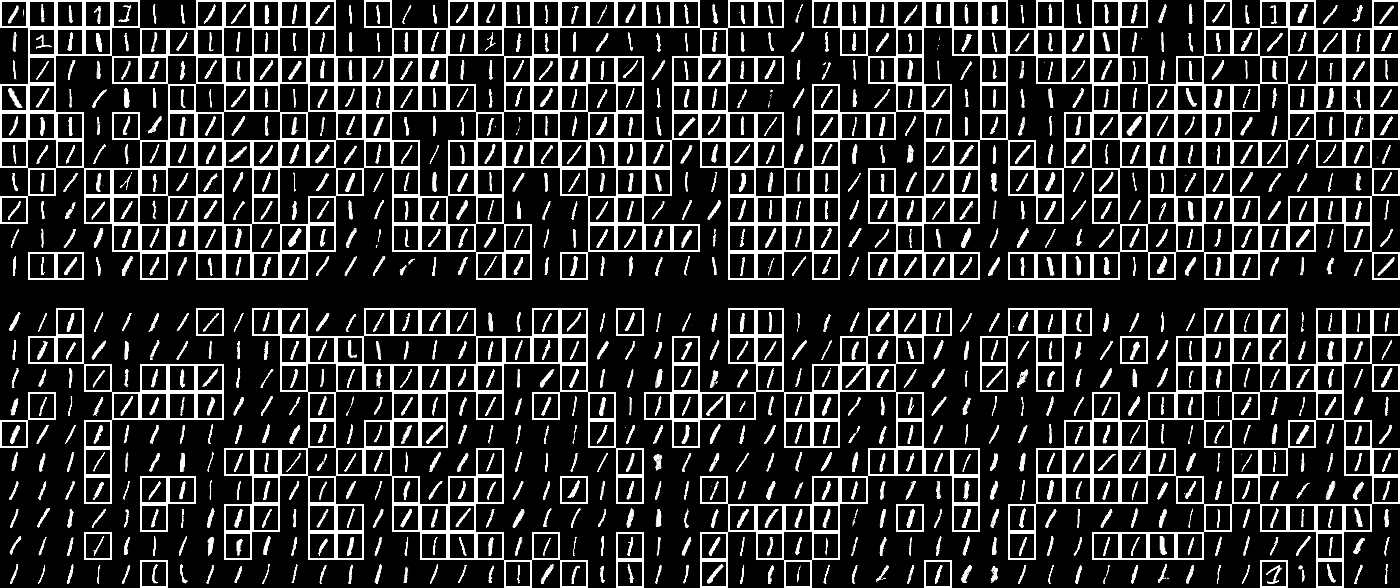}
    \caption{GAN}
    \label{fig:nist_gan_all}
  \end{subfigure}
  \caption{1000 generated and test images (bounding box) ranked according their assigned energies.}
 \label{fig:nist_result_all}
\end{figure}

\subsection{Classifier performance as a proxy measure}
As mentioned in Section \ref{sec:experiments}, evaluating the proposed formulation quantitatively on high-dimensional data is extremely challenging.
Here, in order to provide more quantitative intuitions on the learned discriminator at convergence, we adopt a proxy measure.
Specifically, we take the last-layer activation of the converged discriminator network as \textbf{fixed} pretrained feature, and build a linear classifier upon it. 
Hypothetically, if the discriminator does not degenerate, the extracted last-layer feature should maintain more information about the data points, especially compared to features from degenerated discriminators.
Following this idea, we first train EGAN-Ent-NN, EGAN-Const, and GAN on the MNIST till convergence, and then extract the last-layer activation from their discriminator networks as fixed feature input.
Based on fixed feature, a randomly initialized linear classifier is trained to do classification on MNIST.
Based on 10 runs (with different initialization) of each of the three models, the test classification performance is summarized in Table \ref{tab:proxy_fixed}.
For comparison purpose, we also include a baseline where the input features are extracted from a discriminator network with random weights.

\begin{table}[h!]
\centering
\begin{tabular}{c | c | c | c | c } 
\toprule
Test error (\%) & EGAN-Ent-NN & EGAN-Const & GAN & Random \\
\midrule
Min  & \textbf{1.160} & 1.280 & 1.220 & 3.260 \\
Mean & \textbf{1.190} & 1.338 & 1.259 & 3.409 \\ 
Std. & 0.024          & 0.044 & 0.032 & 0.124 \\
\bottomrule
\end{tabular}
\caption{Test performance of linear classifiers based on last-layer discriminator features.}
\label{tab:proxy_fixed}
\end{table}

Based on the proxy measure, EGAN-Ent-NN seems to maintain more information of data, which suggests that the discriminator from our proposed formulation is more informative.
Despite the positive result, it is important to point out that maintaining information about categories does not necessarily mean maintaining information about the energy (density).
Thus, this proxy measure should be understood cautiously.

\end{document}